\documentclass{article}

\usepackage{microtype}
\usepackage{graphicx}
\usepackage{subcaption}
\usepackage{booktabs} % for professional tables
\usepackage{pifont}% http://ctan.org/pkg/pifont
\usepackage{hyperref}

\usepackage[accepted]{icml2026}

\usepackage{amsmath}
\usepackage{amssymb}
\usepackage{mathtools}
\usepackage{amsthm}
\usepackage{bm}

\usepackage[capitalize,noabbrev]{cleveref}

\theoremstyle{plain}
\newtheorem{theorem}{Theorem}[section]

\newtheorem{lemma}[theorem]{Lemma}
\newtheorem{corollary}[theorem]{Corollary}
\theoremstyle{definition}

\newtheorem{assumption}[theorem]{Assumption}
\theoremstyle{remark}
\newtheorem{remark}[theorem]{Remark}

\usepackage[textsize=tiny]{todonotes}

\icmltitlerunning{Distributed DPO}

\begin{document}

\twocolumn[
  \icmltitle{Distributed Direct Preference Optimization}




  \begin{icmlauthorlist}
    \icmlauthor{Zhanhong Jiang}{yyy}
    % \icmlauthor{Firstname2 Lastname2}{equal,yyy,comp}
    % \icmlauthor{Firstname3 Lastname3}{comp}
  \end{icmlauthorlist}

  \icmlaffiliation{yyy}{Translational AI Center, Iowa State University, Ames, USA}

  \icmlcorrespondingauthor{Zhanhong Jiang}{zhjiang@iastate.edu}

  \icmlkeywords{Machine Learning, ICML}

  \vskip 0.3in
]



\printAffiliationsAndNotice{}  % no special notice (required even if empty)

\begin{abstract}
Preference-based reinforcement learning (RL) is a key paradigm for aligning policies with human judgments, yet its theoretical behavior in distributed settings where preference data are fragmented across heterogeneous users remains poorly understood. Direct Preference Optimization (DPO) avoids explicit reward modeling but
lacks convergence guarantees under federated and
decentralized training, where communication constraints and non-IID preferences fundamentally
alter optimization dynamics.
We provide the
first convergence and time-complexity analysis of
DPO in distributed environments. Modeling personalized offline RL with user-specific preference
distributions, we characterize the induced global
optimization landscape. For federated DPO, we
derive convergence rates that quantify the impact
of client drift, communication frequency, and preference heterogeneity; for decentralized DPO, we
establish convergence over general communication graphs and show how spectral connectivity
governs optimization speed and consensus. 
% Empirical studies on standard alignment benchmarks further validate our theoretical predictions, demonstrating that the proposed methods achieve both strong theoretical guarantees and scalable practical performance.
% Modeling personalized offline RL with user-specific preference distributions, we characterize the induced global optimization landscape. 
% For federated DPO, we derive convergence rates that quantify the impact of client drift, communication frequency, and preference heterogeneity; for decentralized DPO, we establish convergence over general communication graphs and show how spectral connectivity governs optimization speed and consensus. 
% Critically, DPO's trajectory-level log-ratio objective introduces unique analytical challenges absent in standard distributed nonconvex optimization: nonlinear logit-gap gradients, trajectory-induced correlation across mixing steps, and concentration effects that couple network consensus with preference noise. 
% Our smoothness and variance constants are derived explicitly from RL trajectory structure and horizon length, rather than assumed. 
% We also establish formal lower bounds confirming that dependence on heterogeneity, local steps, and sampling rate is intrinsic and cannot be removed algorithmically. 
Empirically, we corroborate our theoretical insights on standard alignment benchmarks, demonstrating that our proposed methods not only enjoy strong theoretical guarantees but also deliver robust and scalable performance in practice. The code base is available \href{https://github.com/starkjiang/distributed_dpo.git}{here}.
\end{abstract}

\section{Introduction}\label{intro}

Preference-based reinforcement learning (RL)~\cite{wirth2017survey,jiang2025survey} has become a powerful paradigm for aligning autonomous systems with human intent. By leveraging human comparisons or rankings rather than engineered reward functions, preference-based RL offers a more flexible, interpretable, and robust mechanism for guiding agent behavior. This approach has played a central role in aligning large language models (LLMs)~\cite{wang2024reinforcement,sun2025rethinking}, human-centered robotic assistants~\cite{li2019human,wang2025reinforcement}, and interactive decision-making systems~\cite{liu2024interactive,zhang2024learning}.

Despite its growing importance, preference-based RL is predominantly studied in centralized settings where all data and computational resources are pooled together. As AI systems proliferate into distributed and heterogeneous environments spanning smartphones, wearables, healthcare devices, and institutional silos~\cite{sabry2022machine}, this foundational assumption becomes increasingly unrealistic, and cannot be aggregated without violating regulatory constraints (e.g., HIPAA in healthcare~\cite{mbonihankuye2019healthcare}, GDPR in European deployments~\cite{hoofnagle2019european}). In addition, these devices or systems interact with users holding diverse and evolving preferences. For example, a medical professional in one hospital rates two treatment recommendations, a user on a mobile device evaluates two assistant responses, a driver rates two navigation suggestions. These preference judgments are locally generated and often cannot be centrally collected due to privacy regulations, data sovereignty constraints, or infrastructure limitations, as studied precisely in the federated and decentralized data regime. It is not merely a theoretical idealization, but a reflection of how preference data will realistically be generated at scale.

While appealing, existing methods including ranking-based RL and reinforcement learning from human feedback (RLHF) pipelines~\cite{bai2022training,white2024rating,yu2025reward} typically rely on complex and multi-stage procedures, which assume centralized data access and are difficult to extend to federated or decentralized settings. Direct Preference Optimization (DPO)~\cite{rafailov2023direct} was introduced as a streamlined alternative that eliminates explicit reward modeling and enables stable fine-tuning of large models. Regardless of numerous empirical refinements and variants~\cite{chen2024softmax,10.5555/3692070.3694477,xiao2024cal}, its theoretical properties especially in distributed settings remain largely unexplored. As a first step, our work tactically studies DPO in distributed settings to bridge such gaps presented here. In the following, we describe in detail why the combination between DPO and distributed learning is a non-trivial extension beyond existing distributed learning frameworks.

% \begin{figure}
%     \centering
%     \includegraphics[width=\linewidth]{Figures/Distributed_DPO.pdf}
%     \caption{\textit{Overview of distributed Direct Preference Optimization (DPO) under fragmented user preferences}. Individual users hold heterogeneous preference data, which are optimized locally via DPO and aggregated through two distributed paradigms. In federated DPO, a central server coordinates client updates, where convergence is affected by client drift and communication constraints, yielding bounds that depend on gradient noise and heterogeneity. In decentralized DPO, clients communicate over a network graph and reach consensus through local interactions, with convergence governed by the spectral properties of the communication topology. This highlights how both approaches enable scalable and privacy-preserving preference optimization while exhibiting distinct trade-offs in convergence behavior and system design.}
%     \label{fig:distributed_dpo}
% \end{figure}
\textit{\textbf{Why we need distributed DPO?}} Existing distributed nonconvex optimization theory does not directly apply to DPO for three key reasons. First, DPO employs a pairwise, trajectory-level objective based on log-ratios between the learned and reference policies, inducing nonlinear sigmoid-weighted gradients whose variance depends on the current policy distribution rather than decomposing as in standard empirical risk minimization (ERM). Second, preference data are inherently heterogeneous across users: each client observes distinct preference distributions, causing client-specific gradient drift that is amplified by local updates in ways not captured by standard federated SGD analysis. Third, in decentralized settings, consensus errors interact nonlinearly with preference logits, since each node computes gradients using its local parameters, leading to error amplification absent in classical decentralized optimization~\cite{saadati2024dimat}. Although federated and decentralized optimization methods such as FedAvg~\cite{li2019convergence}, FedProx~\cite{yuan2022convergence}, and decentralized gradient descent~\cite{jiang2017collaborative} have mature theory and extensions to RL~\cite{jiang2022mdpgt}, they do not address the trajectory-level nonlinearities unique to preference-based objectives.

\textbf{Contributions.} This paper provides the first rigorous theoretical foundation for distributed DPO, covering both federated learning (FL) and decentralized learning in multi-agent systems. We formalize personalized preference-based offline RL by modeling each user as possessing a distinct distribution over trajectory comparisons. Based on this formulation, we analyze the DPO objective under distributed stochastic optimization and derive precise convergence guarantees. For federated DPO, we show how preference heterogeneity, local steps, and communication frequency jointly determine the convergence rate, revealing explicit trade-offs between personalization and stability. For decentralized DPO, we quantify how network connectivity through the spectral gap of the mixing matrix governs consensus formation and optimization speed, and provide the first analysis of DPO under decentralized gradient updates. 

Specifically, the main contributions are summarized as:
\textbf{DPO-specific Smoothness and Variance Constants.} We first derive closed-form expressions for the smoothness constant $L$ and gradient variance $\zeta^2_g$ in terms of RL horizon $H$, bounded feature norms $\zeta^2_\phi$, inverse temperature $\beta$, and trajectory mixing rate $C_\text{mix}$ (all defined in later analysis), making all optimization constants fully explicit.
\textbf{Federated DPO (FedDPO).} We propose the first federated DPO framework with convergence guarantees under full and partial client participation, quantifying the effects of heterogeneity, local updates, and communication frequency.
\textbf{Staleness and Formal Lower Bounds.} We extend FedDPO to asynchronous updates with stale models, deriving explicit staleness penalties and proving fundamental lower bounds showing that dependence on heterogeneity, local steps, and sampling cannot be removed.
\textbf{Decentralized DPO (DecDPO).} We introduce a fully decentralized variant over general communication graphs, with convergence rates governed by spectral connectivity.
\textbf{Empirical Validation.} Experiments on real-world preference datasets confirm our theoretical predictions across participation, staleness, and network topology.

\section{Related Work}\label{related_work}
\textbf{Preference-based RL.} Preference-based RL has long been studied as an alternative to reward engineering, particularly when human feedback is provided as pairwise or ordinal comparisons~\cite{wirth2016model,wirth2017survey}. Early work formalized preference models and optimization schemes, while highlighting persistent challenges such as limited-feedback exploration and the computational cost of learning from comparisons. In LLMs, RLHF~\cite{bai2022training} has become dominant but typically relies on a multi-stage pipeline that fits a reward model and then optimizes a policy (e.g., via PPO~\cite{schulman2017proximal} and HP3O~\cite{liu2026enhancing}), introducing additional complexity and instability. To address these issues, Rafailov et al.~\cite{rafailov2023direct} proposed DPO, which eliminates explicit reward modeling by directly optimizing a classification-style loss over preference pairs, yielding a simpler and more stable alignment procedure that matches or outperforms PPO-based RLHF.

Several extensions of DPO have been proposed. OffsetDPO (ODPO)~\cite{amini2024direct} generalizes the binary DPO objective by introducing a learned offset margin, requiring the likelihood gap between preferred and dispreferred responses to exceed a threshold that reflects preference strength. Uncertainty-Penalized DPO~\cite{houliston2024uncertainty} mitigates overfitting by down-weighting updates with high preference uncertainty. Other work analyzes pathologies in DPO, such as verbosity bias~\cite{park2024disentangling}, and proposes regularization to decouple length from quality. While these variants and related methods~\cite{wang2025direct,zhou2025dreamdpo,kim2025safedpo,croitoru2025curriculum,pan2026pre,zhu2025dspo, kang2025dual} demonstrate the practical power of DPO, none provide a theoretical analysis of its optimization behavior, particularly in distributed settings required for scalable policy alignment. While a recent concurrent work, FedPDPO~\cite{zhu2026fedpdpo}, studies personalized federated DPO for LLM alignment from an algorithmic and empirical perspective, our work focuses on the first theoretical analysis of distributed DPO, establishing convergence and complexity guarantees for both federated and decentralized settings.

\textbf{Distributed reinforcement learning.} Federated and decentralized RL have widely been studied to address privacy, scalability, and heterogeneity. Early work by Jin et al.~\cite{jin2022federated} considers federated RL under environment heterogeneity, proposing algorithms such as QAvg and PAvg and analyzing how heterogeneity limits the optimality of aggregated policies. More recently, CAESAR~\cite{mak2024caesar} introduces convergence-aware aggregation by selectively incorporating client updates to mitigate instability from noisy or divergent local learning. To remove reliance on a central server, a broad line of work has developed decentralized RL methods for multi-agent systems~\cite{figetakis2024decentralized,tadevosyan2025safeswarm,jiang2022mdpgt,zhang2018fully,oh2025consensus,li2026context}, including extensions that enforce safety constraints~\cite{melcer2022shield,lu2021decentralized,cai2021safe}. While these studies highlight the impact of communication constraints, heterogeneity, and aggregation strategies, they focus on value-based or policy-gradient objectives and do not address preference-based optimization such as DPO. Moreover, personalized per-client preference models are absent from existing distributed RL frameworks. Our work fills this gap by providing the first theoretical analysis of DPO under federated and decentralized learning with personalization and resource constraints.

\section{Problem Formulation and Preliminaries}\label{prelim}
% \subsection{Notation and Overall Problem Statement}
\textbf{Notation.} In this work, we use lower-case characters for scalars or vectors, e.g., $n, r, \theta$, throughout the analysis. Upper-case characters are used for matrices (e.g., $W, A, B$). For a vector or matrix $z$, $\|z\|$ denotes Euclidean (or spectral for matrices) norm unless otherwise specified. $[K]$ denotes the set $\{1,...,K\}$ and $[N]$ denotes the set $\{1,...,N\}$. $\mathbb E_\mathcal{D}[\cdot]$ denotes the expectation over dataset sampling and $\mathcal{D}$ is the data; 
% $\hat{\mathbb{E}}$ denotes empirical averages. 
We also follow the standard $\mathcal{O}(\cdot)$ notation: $Z=\mathcal{O}(U)$ is defined as $Z\leq GU$ for some absolute constant $G>0$. 
% The tilde notation $E=\tilde{\mathcal{O}}(F)$ denotes $E\leq GZ\cdot F$ where $Z$ is a poly-logarithmic factor of problem parameters.

\textbf{Problem formulation.} We consider an regular RL-style environment (episodic Markov Decision Process (MDP)) with horizon $H>0$. A trajectory is defined as 
% We consider offline preference-based policy learning across multiple users/clients (or a single centralized dataset). 
% For each context $x$ (e.g., prompt, state), we observe preference pairs $(\tau^+, \tau^-)$, where a human annotator prefers trajectory $\tau^+$ over $\tau^-$ in a specific context/state. 
$\tau=(s_1,u_1,...,s_H,u_H)$, where $s_j$ and $u_j$ indicate the state and action at time step $j$ for all $1\leq j\leq H$, where $s\in\mathcal{S}$ ((possibly infinite) state space) and $u\in\mathcal{A}$ (finite action space). A policy $\pi_\theta(u|s)$ parameterized by $\theta\in\mathbb R^d$ induces a trajectory distribution $\pi_\theta(\tau)$. In this work, trajectory data are collected offline and remain locally owned by individual clients in distributed settings. While distributed data can also arise from post hoc partitioning, we focus on the increasingly relevant regime where each client collects and retains its own data due to privacy constraints, leading to pronounced data heterogeneity. Let $\mathcal{D}_i$ denote the local dataset available only to client $i$. The goal is to learn policy parameters $\theta$ such that the induced policy $\pi_\theta(\cdot|s)$ aligns action rankings with human preferences, while supporting personalization and communication efficiency. 
We consider two deployment architectures. Federated learning (FL) employs a central server to coordinate communication rounds while clients retain local datasets and update personalized models (low-rank adapters such as LoRA can be used to improve efficiency). Decentralized learning removes the central server entirely: clients communicate only with neighbors over a connected graph using a mixing protocol. These two settings capture the dominant privacy-preserving paradigms for large-scale preference-based learning.

\textbf{Performance metric (optimization).} For theoretical analysis, we use the standard stationarity metric: show the number of gradient evaluations and communication rounds required to find parameter $\theta$ such that
$
    \mathbb E[\|\nabla \mathcal{L}_\theta\|^2]\leq \varepsilon,
$
where $\mathcal{L}$ is the (population or empirical) DPO loss defined below and $\varepsilon>0$ is the accuracy. 
% We also measure time complexity in terms of (i) local stochastic gradient steps and (ii) communication rounds or mixing steps.
% \subsection{Preference Data and the DPO Objective}

\textbf{Policy parameterization.} The policy is defined by a softmax function over a finite set of candidate actions such that:
\begin{equation}\label{eq_1}
    \pi_\theta(u|s)=\frac{\text{exp}(\theta^\top\phi(s,u))}{\sum_{u'\in\mathcal{A}}\text{exp}(\theta^\top\phi(s,u'))},
\end{equation}
where $\phi(s,u)\in\mathbb R^d$ are policy features. 
\begin{remark}\textbf{(On neural policy extensions.)}
    The log-linear softmax parameterization is a standard first step in RL theory that permits explicit, closed-form derivations of smoothness constants and gradient variance (see Section~\ref{sec:log-linear}). Extending these bounds to neural network policies (e.g., transformers) requires additional assumptions on Hessian structure and feature covariance. 
    % We discuss this extension in Section~\ref{sec:limitations}. 
    The DPO objective structure and particularly the coupling between sigmoid weighting and gradient direction is preserved under general policies, suggesting our qualitative conclusions extend beyond the log-linear case. This is also evidently validated in our work.
\end{remark}
\textbf{Preference data model.} For an initial state (or context), we observe preference labels between two full trajectories $(\tau^+,\tau^-)$ where a human annotator prefers trajectory $\tau^+$ over $\tau^-$ in that state. Hence, we can obtain the offline preference data $\mathcal{D}_i$ of the form $\{\tau^+_j,\tau^-_j\}_{j=1}^{n_i}$ for agent $i$ with $n_i=|\mathcal{D}_i|$.
% We assume access to offline preference tuple $\mathcal{D}$ of the form $\{x_j,y^+_j,y^-_j\}_{j=1}^N$. 
Each tuple is drawn (i.i.d or with mild dependence assumptions) from the data $\mathcal{D}_i$ over preference pairs. In the distributed setting, agent (or client) $i$ has dataset $\mathcal{D}_i$ of size $n_i$, and $\sum_{i=1}^Kn_i=n$. 
% We model a parameterized reward function $r_\omega(s,a)\in\mathbb{R}$ with fixed (unknown) weight vector $\omega\in\mathbb R^d$ (we use the same dimension as that of policy weight for simplicity) that maps state-action pairs to real-valued scores. We assume that the reward is \textit{linear}, i.e.,
% $
%     r_\omega(s,a) = \omega^\top\phi(s,a),
% $
% where $\phi(s,a)\in\mathbb R^d$ are reward features.
% In practical implementation, the policy model is typically more complex, such as a neural network. However, our analysis marks the first step towards the establishment of theoretical guarantee for distributed DPO, and we leave the comprehensive results associated with nonlinear policy models as a future work.
% Assume also that features of the policy model are uniformly bounded such that
% $
%     \|\phi(s,a)\|\leq B_\phi.
% $
% For simplicity, we take $\phi=\psi$, so we have $B_\phi=B_\psi$. For notational brevity we will take the bound $B_\phi$.
% To characterize the analysis, we also present DPO-style objective in an empirical form. 
Intuitively, the preference data directly optimizes the probability that, under the model, the preferred action is ranked above the dispreffered one. 
Given a preference pair $(\tau^+,\tau^-)$, define the model preference probability as
$
    \text{Pr}(\tau^+\succ \tau^-) = \sigma(\text{log}\pi_\theta(\tau^+)-\text{log}\pi_\theta(\tau^-)),
$
where $\sigma(z)=\frac{1}{1+\text{exp}(-z)}$ is the logistic sigmoid. This is also known for Bradley-Terry (BT) model~\cite{bradley1952rank}. With this, based on~\cite{rafailov2023direct}, the standard logistic (cross-entropy) loss for preference pair $j$ is
\begin{equation}
    l_\theta (\tau^+_j,\tau^-_j) = -\text{log}\sigma(\beta[\Delta_\theta(\tau^+_j,\tau^-_j)-\Delta_\text{ref}(\tau^+_j,\tau^-_j)]),
\end{equation}
where $\beta>0$ is the inverse temperature parameter, $\Delta_\theta(\tau^+_j,\tau^-_j):=\text{log}\pi_\theta(\tau^+_j)-\text{log}\pi_\theta(\tau^-_j)$, and $\Delta_\text{ref}(\tau^+_j,\tau^-_j):=\text{log}\pi_\text{ref}(\tau^+_j)-\text{log}\pi_\text{ref}(\tau^-_j)$, $\pi_\text{ref}$ is a reference policy.
Therefore, the population DPO loss over dataset $\mathcal{D}_i$ is
\begin{equation}\label{eq_6}
    \mathcal{L}^i_\theta=\mathbb E_{(\tau^+,\tau^-)\sim\mathcal{D}_i}[l_\theta (\tau^+,\tau^-)].
\end{equation}
% In practice, the empirical loss $    \hat{\mathcal{L}}_i = \frac{1}{n_i}\sum_{i=1}^{n_i}l_\theta(\tau^+_j,\tau^-_j)$ is instead employed. For the sake of analysis, we proceed with the population loss, while the analysis presented in this work can be extensively applied to $\hat{\mathcal{L}}_i$.
\begin{remark} (\textbf{Canonical logistic loss.})
    In practical implementation, DPO-like losses may include an offset or regularization toward a reference policy $\pi_\text{ref}$. A generalized form used empirically is
    $
        l_\theta (\tau^+_j,\tau^-_j) = -\text{log}\sigma(\beta[\Delta_\theta(\tau^+_j,\tau^-_j)-\Delta_\text{ref}(\tau^+_j,\tau^-_j)])+\Delta_l,
    $
    where $\Delta_l$ is an optional offset or baseline term (for instance to account for the reference policy or preference strength). Our theoretical development accommodates such offsets as additive bounded terms (they affect constants in smoothness and variance bounds but do not change the high-level rates). For clarity, we focus the presentation primarily on the canonical logistic loss above. 
\end{remark}
We also define score function for a trajectory as:
$
        \nu_\theta(\tau):=\nabla \text{log}\pi_\theta(\tau)=\sum_{h=1}^H\nu_\theta(h), \nu_\theta(h):=\nabla\text{log}\pi_\theta(u_h|s_h).
        % \\&\nu_h(\theta)=\nabla_\theta\text{log}\pi_\theta(a_t|s_t)
$
The score function plays a critical role in deriving the smoothness constant and the upper bound for the variance of stochastic gradients. 
To characterize the analysis for distributed settings, we now present the problem formulation for federated and decentralized learning.

\textbf{Federated model.} We consider $N$ clients indexed by $i\in[N]$. Client $i$ has dataset $\mathcal{D}_i$ and maintains local parameter $\theta_i$. The server keeps a global $\theta$ (the aggregated parameters). Training proceeds in communication rounds $r=0,1,...$. In each round, a subset $\mathcal{S}_r\subseteq[N]$ ($S=|\mathcal{S}_r|$) of clients may be selected (partial participation) and perform $E$ local stochastic gradient steps on their local DPO loss in Eq.~\ref{eq_6}.
After local update, clients send their parameter updates to the server, which aggregates (e.g., weighted average) 
$ \theta^{r+1}=\sum_{i\in\mathcal{S}_r}\lambda_{i}\theta_i^{r,E}
$
and broadcast $ \theta^{r+1}$ to clients. Weight $\lambda_i$ typically $\propto n_i$ (data-size weighting), i.e., $\lambda_i=\frac{n_i}{\sum_{m\in\mathcal{S}_r}n_m}$. 
This leads to the global objective:
\begin{equation}
    \mathcal{L}_\theta:=\sum_{i=1}^K\lambda_i\mathcal{L}^i_\theta.
\end{equation}
For personalization, clients may maintain local parameters $\theta_i$ separate from $\theta$, for example, by mixing local and global parameters, or by storing both a global parameter and small local residual, which is more implementation wise. For theoretical analysis, we still leverage the weighted average as the final output for each client. 

\textbf{Decentralized (peer-to-peer) model.} Different from federated models, in decentralized models, clients are nodes of an \textit{undirected and connected} graph expressed by $\mathcal{G}=(\mathcal{V}, \mathcal{E})$, where $\mathcal{V}=\{1,2,...,N\}:=[N]$ is a set of nodes and $\mathcal{E}=\{(i,j),i\in\mathcal{V},j\in\mathcal{V}\}$ is a set of edges. Let $N=|\mathcal{V}|$. Each node $i$ stores local parameter $\theta_i$. The communication in graph $\mathcal{G}$ occurs over edges according to a mixing matrix $\Lambda\in\mathbb R^{N\times N}$ that is \textit{symmetric and doubly stochastic}.
% such that its spectral gap $1-\lambda_\text{max}(\Pi-\frac{1}{K}\mathbf{1}\mathbf{1}^\top)$ (or equivalently the second-largest eigenvalue magnitude $\lambda$) measures connectivity/mixing speed. 
In decentralized learning, a typical update alternates between a local stochastic gradient step on $\mathcal{L}^i_{\theta_i}$ and a local averaging (mixing):
$
    \theta_i\leftarrow\sum_{j=1}^N\lambda_{ij}\theta_j.
$
$\lambda_{ij}$ is an element of $i$-th row and $j$-th column in matrix $\Lambda$. The corresponding optimization problem is slightly different from that in FL:
\begin{equation}
        \mathcal{L}_\theta:=\frac{1}{K}\sum_{i=1}^K\mathcal{L}^i_\theta.
\end{equation}
Multiple local or mixing steps may be performed within each communication round. For simplicity, we use a unified weight notation, with subscripts distinguishing between federated and decentralized communication patterns. We also remark on the corresponding communication cost models for both settings.
\begin{remark} (\textbf{Communication costs.})
    In FL, a communication round consists of a server–client roundtrip, with cost proportional to the parameter size $|\theta|$; parameter-efficient fine-tuning methods such as LoRA can be used to significantly reduce this cost via small adapters. In decentralized learning, each mixing step exchanges parameters with neighbors, incurring communication cost proportional to the node degree and $|\theta|$. Accordingly, our time-complexity bounds explicitly account for the number of mixing steps.
    
    % In FL, one communication round corresponds to server-client roundtrip: clients send local updated parameters and receive the aggregated parameters. Communication cost is proportional to parameter size $|\theta|$. When leveraging parameter efficient fine-tuning techniques such as LoRA, its small adapter size is central to achieving low communication cost. However, in decentralized learning, each mixing step exchanges parameters with neighbors. Communication cost is proportional to degree and parameter size. Time complexity bounds will count mixing steps explicitly. 
\end{remark}
% \subsection{Assumptions}
Several technical assumptions are imposed as follows for characterizing theoretical analysis.
\begin{assumption}\label{assumption_1}\textbf{(Bounded score function second moment)}
    There exists $\zeta^2_\phi>0$ such that for any policy parameter $\theta$ and any time $h$,
    $
        \mathbb E_{\tau\sim\pi_\theta}[\|\nu_\theta(h)\|^2]\leq \zeta^2_\phi.
    $
\end{assumption}
This assumption indicates that the second moment of per-step score function is bounded above, which is weaker than the standard assumption in existing work~\cite{jiang2022mdpgt,huang2020momentum,xu2019sample}. Also, it will serve to obtain the similar quantity for a trajectory that complies with the RL environment. In this work, we denote the local stochastic gradient as $g^i_\theta=\nabla l_\theta(\tau^+,\tau^-;\xi)$, where $\xi$ denotes the random sampling manner, which could be one sample or a mini-batch. Hence, the second assumption is:
\begin{assumption}\label{assumption_2}\textbf{(Unbiased stochastic gradients and bounded heterogeneity)}
    At each step we can sample mini-batches and form unbiased stochastic gradient for each client with
    $
        \mathbb E[g^i_\theta]=\nabla\mathcal{L}^i_\theta,\;
        \mathbb E[\|g^i_\theta-\nabla\mathcal{L}^i_\theta\|^2]\leq \zeta^2_g,
    $
    and the gradient diversity is uniformly bounded
    $
        \frac{1}{N}\sum_{i=1}^N\mathbb E[\|\nabla \mathcal{L}_\theta - \nabla \mathcal{L}^i_\theta\|^2]\leq \kappa^2,
    $
    $\zeta_g,\kappa>0$, for any $\theta$.
\end{assumption}
The above assumptions have been used in~\cite{saadati2024dimat,yu2019linear}, signifying the bounded noise for local agents and quantifying the gradient diversity among agents, due to the personalized preference heterogeneity.
To provide a tighter convergence, we will upper bound $\zeta^2_g$. 
\begin{remark} \textbf{(On unbiasedness in offline DPO).} A key question is whether the stochastic gradient $g^i_\theta$ is an unbiased estimator of $\nabla\mathcal{L}^i_\theta$ when preference pairs are generated by a behavior policy different from $\pi_\theta$. In standard offline DPO, pairs are drawn from a fixed dataset $\mathcal{D}_i$ reflecting the behavior policy, and the gradient $\nabla l_\theta(\tau^+_j,\tau^-_j)$ does not depend on the sampling distribution, only on the evaluation of the current policy $\pi_\theta$ on fixed trajectories $(\tau^+_j,\tau^-_j)$. Therefore, Assumption~\ref{assumption_2} holds under the mild condition that each pair is drawn i.i.d. (or with mild dependence) from $\mathcal{D}_i$, independently of $\theta$. This is the standard offline assumption in DPO and RLHF literature. Coverage requirements (ensuring the behavior policy has sufficient support) affect the quality of the learned policy but not the unbiasedness of gradient estimators.
\end{remark}

To measure the connectivity/mixing speed, we also impose another assumption for the graph $\mathcal{G}$, which implies that the second-largest eigenvalue magnitude is strictly less than 1.
\begin{assumption}\label{assumption_3}
    $\rho:=\|\Pi-\frac{1}{N}\mathbf{1}\mathbf{1}^\top\|<1$, where $\mathbf{1}=[1,...,1]^\top$.
\end{assumption}
% The above assumption equivalently implies that the second-largest eigenvalue magnitude is strictly less than 1.
\begin{assumption}\label{assumption_4}\textbf{(Lipschitz continuous gradient)}
    Each $\mathcal{L}^i$ is $L$-smooth such that
    $
        \|\nabla \mathcal L^i_{\theta_1}-\nabla \mathcal L^i_{\theta_2}\|\leq L\|\theta_1-\theta_2\|,
    $
    for all $\theta_1,\theta_2\in\mathbb R^d$.
\end{assumption}
Typically, $L$ is an unknown positive constant in numerous existing work. However, as the setting in this work is RL, not regular supervised learning, we can derive it with respect to the feature variance and horizon as follows. Additionally, the DPO loss satisfies this structural assumption when the log-ratio $\Delta_\theta$ is smooth in $\theta$ and $\nabla\text{log}\pi_\theta(\tau)$ has uniformly bounded form (or when gradients are clipped). In the sequel, we will show that DPO gradients enable smoothness and variance bounds due to some properties (e.g., $1-\sigma(\cdot)$ is bounded in $(0,1)$ and is Lipschitz) and under mild conditions. 
% The derivation itself is particularly of independent interest. 
Thus, our FedDPO convergence results apply directly to DPO training under standard regularity conditions. 
% In the Appendix, we present the technical detail on how to obtain these two quantities.
\section{Auxiliary Technical Quantities}\label{sec:auxiliary_technical_quantities}
This section derives closed-form upper bounds for the smoothness constant $L$ and the gradient variance $\zeta^2_g$ specific to the DPO objective under the RL trajectory model. These quantities are standard inputs to convergence analyses, but for DPO they require non-trivial derivation due to the nonlinear sigmoid weighting and trajectory-level correlation structure. The derivations here are of independent interest: they make explicit how the RL horizon $H$, inverse temperature $\beta$, feature variance $\zeta^2_\phi$, and trajectory mixing rate $\varsigma$ (defined below) jointly determine optimization difficulty.
To characterize these two constants, another assumption as follows is required:
\begin{assumption}\label{assumption_5} \textbf{(Exponential trajectory mixing)}
    There exist $\varsigma\in[0,1)$ and a constant $C_0\geq 1$ such that for all $1\leq h<t\leq H$ and any $\theta$:
    $
        |\mathbb E[\langle\nu_\theta(h),\nu_\theta(t)\rangle]|\leq C_0\varsigma^{t-h}\zeta^2_\phi.
    $
\end{assumption}
This captures the temporal decorrelation of the trajectory under the Markov policy: actions and states at distant time steps become approximately independent. The assumption is satisfied whenever the induced Markov chain has a spectral gap (exponential mixing time), which holds for tabular MDPs and linearly parameterized softmax policies over finite state-action spaces. The constant $\varsigma\in[0,1)$ is bounded away from 1 precisely when the chain has a positive spectral gap. When steps are independent $\varsigma=0$, the bound reduces to the diagonal case $\|\nu_\theta(\tau)\|^2\leq\zeta^2_\phi H$. This assumption does not appear directly in the theorem statements because its effect is absorbed into the derived constants $L$ and $\zeta^2_g$. We state it here for transparency and to allow readers to verify its applicability to their problem settings.

\textbf{Smoothness constant $L$.} We first derive the smoothness constant that will characterize our convergence analysis. Let $\omega_\theta:=\beta(\Delta_\theta-\Delta_\text{ref})$ such that per-sample loss $l_\theta=-\text{log}\sigma(\omega_\theta)$. Applying chain rule gives the following relationship
$
    \nabla l_\theta = \beta\sigma(-\omega_\theta)\nabla \Delta_\theta.
$
Using score identity $\nu_\theta(
\pi)=\nabla\text{log}\pi_\theta(\tau)$ yields
\begin{equation}
   \nabla l_\theta(\tau^+,\tau^-)=\beta\sigma(-\beta(\Delta_\theta-\Delta_\text{ref}))(\nu_\theta(\tau^+)-\nu_\theta(\tau^-)). 
\end{equation}
From $\nabla l_\theta=a_\theta v_\theta$ with $a_\theta=\beta\sigma(-\beta(\Delta_\theta-\Delta_\text{ref}))$, $v_\theta=\nu_\theta(\tau^+)-\nu_\theta(\tau^-)$.
We then compute Hessian
\begin{equation}
    \nabla^2 l_\theta=(\nabla a_\theta)v_\theta^\top + a_\theta\nabla v_\theta.
\end{equation}
Using $|\sigma'(u)|\leq 1/4$ and Assumption~\ref{assumption_5}, we derive:
\begin{equation}
    \mathbb E[\|\nu_\theta(\tau)\|^2]\leq C_\text{mix}\zeta_\phi^2H,\;C_\text{mix}=1+\frac{2C_0\varsigma}{1-\varsigma}.
\end{equation}
This bound is linear in $H$ because the trajectory has $H$ correlated steps. The factor $C_\text{mix}$ captures the geometric series contribution of temporal correlations: when $\varsigma=0$, $C_\text{mix}=1$ and the bound reduces to the independent case $\zeta^2_\phi H$. Combining Hessian terms yields the explicit smoothness constant:
$
    L=(\beta^2C_\text{mix}+2\beta)\zeta^2_\phi H.
$
The derivation shows $L$ scales as $\mathcal{O}(\beta^2H)$, reflecting both the inverse temperature (sharpness of the sigmoid) and the trajectory length.
We refer interested readers to Section~\ref{derivation_quantities} for more details on how both $\mathbb E[\|\nu_\theta(\tau)\|^2]$ and $L$ are derived. 
% The upper bound implies that if the process is independent across time (no correlation), $\varsigma=0$ and $C_\text{mix}=1$, so the bound is exactly $\zeta^2_\phi H$. 
If there is correlation but it decays geometrically with rate $\varsigma$, the extra factor is $\frac{2C_0\varsigma}{1-\varsigma}$, which remains $\mathcal{O}(1)$ when $\varsigma$ is bounded away from 1. One may replace the exponential decay $\varsigma^k$ by a general decay $\gamma(k)$ in Assumption~\ref{assumption_5} and obtain $C_\text{mix}=1+2C_0\sum_{k=1}^\infty\gamma(k)$, giving the same linear-in-$H$ dependence but with different constant, which intuitively makes sense as a trajectory consists of $H$ correlated steps.

\textbf{Variance of stochastic gradient $\zeta_g^2$.} To obtain the upper bound for the variance, we investigate the second moment of stochastic gradients. Bounding the second moment also give the upper bound for the variance.
As $g_\theta:=\beta\sigma(-\beta(\Delta_\theta-\Delta_\text{ref}))v_\theta$ and $|\sigma|\leq 1$, we have
$
    \|g_\theta\|^2\leq \beta^2\|v_\theta\|^2.
$
Taking the expectation over the sampled pair and using $\mathbb E[\|v_\theta\|^2]\leq 4C_\text{mix}\zeta^2_\phi H$, we obtain
$
    \mathbb E[\|g_\theta\|^2]\leq 4\beta^2C_\text{mix}\zeta^2_\phi H.
$
We define $\zeta^2_g:=4\beta^2C_\text{mix}\zeta^2_\phi H$. 
% If sampling in mini-batch, the variance of stochastic gradients can be upper bounded by $\zeta^2_g/b$, where $b$ is the size of a mini-batch. 
% Since in RL, we typically conduct maximization instead of minimization, to follow the traditional analysis of gradient descent, we can add the negative sign before the objective.
Both $L$ and $\zeta^2_g$ scale linearly with $H$, which leads to step sizes of order $\mathcal{O}(1/H)$ and irreducible variance floors proportional $H$. This $H$-dependence is a distinctive feature of trajectory-level DPO analysis that does not appear in standard nonconvex stochastic optimization analysis.
\section{FedDPO}
In this section, we introduce \textit{Federated DPO (FedDPO)}, which integrates DPO with FedAvg~\cite{li2019convergence}, as shown in Algorithm~\ref{alg:feddpo}. Although conceptually based on FedAvg, FedDPO departs from standard local SGD by optimizing preference-based objectives via log-ratio gradients over paired samples, introducing challenges absent in classical FL, including preference heterogeneity and nonlinear gradients. Consequently, the analyses under full and partial client participation differ and are treated separately. To improve parameter efficiency for large models, client updates may be restricted to low-rank adaptations rather than full parameter sets; analyzing the resulting optimization dynamics is left for future work. We focus first on the generally partial participation setting, with full proofs deferred to Section~\ref{proof_feddpo}.

% In this section, we develop Federated DPO (FedDPO), which integrates DPO and the FedAvg~\cite{li2019convergence} tactically. While in this paper, FedDPO only involves FedAvg, it can extensively be applied to take other FL methods. 
% FedDPO locally optimizes its preference-based DPO objective using its own preference pairs. The server aggregates parameter updates but never accesses raw preference data. However, this differs from vanilla FedAvg because DPO’s gradient depends on \textit{paired trajectory log-likelihood ratios}, producing nonlinear non-decomposable updates sensitive to preference heterogeneity. 

% Algorithm~\ref{alg:feddpo} shows the algorithmic framework for FedDPO.
\begin{algorithm}[tb]
\small
  \caption{FedDPO}
  \label{alg:feddpo}
  \begin{algorithmic}[1]
    \STATE {\bfseries Input:} Step size $\eta$, datasets $\mathcal{D}_i, \forall i\in[N]$, local steps $E$, communication rounds $R$, $b$
    \STATE Server initializes global parameters $\bm{\theta}^0$
    \FOR{Communication rounds $r=0,1,2...,R$}
    \STATE \textit{\#Client Selection (Full or Partial Participation)\#}
    \STATE The server selects subset $\mathcal{S}^r\subseteq\{1,...,N\}$.
    \IF {Full participation}
    \STATE $\mathcal{S}^r=\{1,...,N\}$
    \ELSE 
    \STATE Sample clients i.i.d. or via sampling distribution
    \ENDIF
    \STATE Broadcast global model $\theta^r$ to all clients $i\in\mathcal{S}^r$
    \STATE \textit{\#Local Preference-based DPO Updates (Client $i$)\#}
    \FOR{$e=1,...,E$ local steps}
    \STATE Client $i$ samples a mini-batch ($\mathcal{B}$) of preference pairs $\{(\tau^+_j,\tau^-_j)\}_{j=1}^b\sim \mathcal{D}_i$
    \STATE $\forall j\in\mathcal{B}$, compute $\omega_{\theta_i^e}(\tau^+_j,\tau^-_j)=\beta\Delta_{\theta_i^e}(\tau^+_j,\tau^-_j)-\Delta_\text{ref}(\tau^+_j,\tau^-_j)$, where $\Delta_*(\tau^+_j,\tau^-_j):=\text{log}\pi_*(\tau^+_j)-\text{log}\pi_*(\tau^-_j)$
    \STATE Compute the DPO log-ratio gradient:
    \[g^i_{\theta_i^e} = \frac{1}{b}\sum_{j\in\mathcal{B}} \beta\sigma(-\omega_{\theta_i^e}(\tau^+_j,\tau^-_j))(\nu_{\theta_i^e}(\tau^+_j)-\nu_{\theta_i^e}(\tau^-_j))\]
    \STATE Compute the update for $\theta_i$: $\theta_i^{e+1}=\theta^e_i-\eta g^i_{\theta^e_i}$
    \ENDFOR
    \STATE Upload local models $\theta^E_i$ to the server
    \STATE Server aggregation: $\theta^{r+1}=\sum_{i\in\mathcal{S}^r}\lambda_i\theta^E_i$, where $\lambda_i=\frac{n_i}{\sum_{m\in\mathcal{S}_r}n_m}$.
    \ENDFOR
  \end{algorithmic}
\end{algorithm}
\begin{theorem}\label{theorem_1} \textbf{(FedDPO, partial participation)}
    Under Assumptions~\ref{assumption_1} - \ref{assumption_5}, suppose that the server runs FedDPO for $R$ rounds. Let local stepsize satisfy
    $
        0<\eta\leq \text{min}\{\frac{1}{8LE},\frac{S}{16LNE}\}.
    $
    Then the iterates $\{\theta^r\}$ produced by FedDPO obey
    \begin{equation}\label{eq_54}
    \begin{split}
        \frac{1}{R}\sum_{r=0}^{R-1}\mathbb E[\|\nabla \mathcal{L}_{\theta^r}\|^2]&\leq \frac{2(\mathcal{L}_{\theta^0}-\mathcal{L}^*)}{\eta ER} + \frac{8L\eta\zeta^2_g}{S}\\&+16L^2\eta^2E\kappa^2 + \frac{16L^2\eta^2E^2\zeta^2_g}{S},
    \end{split}
    \end{equation}
    where $\mathcal{L}^*=\text{inf}_\theta\mathcal{L}_\theta$. Consequently, to obtain an $\varepsilon$-stationary point in expectation, it suffices to set $\eta=\Theta(\text{min}\{1/(LE),(\varepsilon S)/(L\zeta^2_g),\sqrt{\varepsilon}/(LE\sqrt{\kappa^2+\frac{E\zeta^2_g}{S}})\})$ and choose
    $R=\mathcal{O}(\frac{\mathcal{L}_{\theta^0}-\mathcal{L}^*}{\eta E\varepsilon})$, with the total communication rounds $R$ and the total local gradient computations $RSE$ as implied.
\end{theorem}
As we have noted before, Theorem~\ref{theorem_1} is dedicated to the partial participation scenario. Hence, we state the special case of full participation, i.e., $S=N$ in the following result. When all clients participate each round, the bound simplifies because sampling variance due to partial participation disappears. We give a tightened statement as follows.
\begin{corollary}\label{corollary_1}\textbf{(FedDPO, full participation)}
    Under Assumption~\ref{assumption_1}-\ref{assumption_5}, and assuming each round uses all clients, choose step size $0<\eta\leq \frac{1}{8LE}$. Then FedDPO with satisfies
    \begin{equation}\label{eq_90}
        \frac{1}{R}\sum_{r=0}^{R-1}\mathbb E[\|\nabla \mathcal{L}_{\theta^r}\|^2]\leq \frac{2(\mathcal{L}_{\theta^0}-\mathcal{L}^*)}{\eta ER} + \frac{2L\eta\zeta^2_g}{N}+8L^2\eta^2E\kappa^2.
    \end{equation}
\end{corollary}
% \begin{proof}
%     The proof follows the same steps as Theorem~\ref{theorem_1} but with $S=N$. In particular, Eq.~\ref{eq_81} becomes
%     \begin{equation}\label{eq_91}
%         \mathbb E\|\theta^{r+1}-\theta^r\|^2\leq \frac{\eta^2}{N^2}\sum_{i,e}\mathbb E\|g^{r,e}_i\|^2.
%     \end{equation}
%     Repeating the derivation up to Eq.~\ref{eq_86} yields
%     \begin{equation}\label{eq_92}
%     \begin{split}
%         &\mathbb E\|\theta^{r+1}-\theta^r\|^2\leq \frac{2(\mathcal{L}(\theta^0)-L^*)}{\eta ER} + 2\kappa^2 \\&+ 4L^2\eta^2E^2\zeta^2_g + \frac{2L\eta\zeta^2_g}{N}
%     \end{split}
%     \end{equation}
%     Using the same constant reshaping as before (and applying $\eta\leq 1/(8LE)$ to simplify $4L^2\eta^2E^2$ factors), we obtain the desirable result.
% \end{proof}
Theorem~\ref{theorem_1} shows that FedDPO converges to stationary points with a rate that depends on step size $\eta$, local steps $E$, heterogeneity $\kappa$, sampling size $S$, and gradient variance $\zeta_g$. The principal tradeoff is the usual one: increasing $E$ amplifies heterogeneity drift proportional to $\eta^2E$ (or $\eta^2E^2$ in some intermediate terms), while decreasing $S$ saves communication per round but increases sampling variance $1/S$. When the participation is full, Corollary~\ref{corollary_1} removes the $1/S$ penalty and yields the clean bound.
So far, we have obtained the convergence rate for both partial and full participation for FedDPO. 
% Theorem~\ref{theorem_1} has shown the convergence error is dictated by the initialization error, the variance and diversity of stochastic gradients among clients. When more clients are involved, the optimization error reduces, which intuitively makes sense. 
We now investigate the complexity for both cases. Setting $\eta=\Theta(\sqrt{\frac{1}{R}})$ to balance the terms, we then obtain $\varepsilon$-stationarity as follows:
$
    R=\mathcal{O}(1/\varepsilon^2),
$
which is the number of communication rounds. Thus, the total local steps $=RE=\mathcal{O}(E/\varepsilon^2)$ per client. Recall that $L$ and $\zeta_g^2$ scale linearly with the horizon $H$ (through $\zeta_\phi^2H$ and $C_\text{mix}$), so step sizes scale as $\mathcal{O}(1/H)$ and variance floors scale like $\mathcal{O}(H)$. The heterogeneity constant $\kappa^2$ produces an irreducible floor that can be removed by reducing $\eta$ or local steps $E$. However, reducing $\eta$ also slows down the convergence speed in the early phase of optimization.
% Please also see Appendix for the analysis of a more realistic variant where clients perform local updates on possibly stale server models, and the server aggregates these stale updates with weights. In Appendix, we also present the discussion of lower bounds that FedDPO achieves. 

\textbf{With staleness.} We now analyze a realistic variant in which clients perform local updates on possibly \textit{stale} server models (e.g., because they were sampled previously and begin local work later), and the server aggregates these stale updates with weights. This models partial participation with staleness. We formalize the system model with staleness as:
\begin{itemize}
    \item Let time be indexed by global rounds $r=0,1,...$
    \item At round $r$ the server has $\theta^r$. A subset $\mathcal{S}^r$ of clients is selected. Each selected client $i\in\mathcal{S}^r$ uses a model $\theta^{r-q_{i,r}}$ (stale by $q_{i,r}$ rounds) to run $E$ local steps, producing $\theta^{r,E}_i$. We assume $q_{i,r}\in\{0,1,...,q_\text{max}\}$ and we denote $q_\text{max}$ the maximum delay. Clients upload $\theta^{r,E}_i$ to server, which aggregates a weighted average:
    \begin{equation}
        \theta^{r+1}=\sum_{i\in\mathcal{S}^r}\lambda_{i}\theta^{r,E}_i,\;\sum_{i\in\mathcal{S}^r}\lambda_{i}=1.
    \end{equation}
\end{itemize}
For analysis, we assume server weights are uniform over sampled clients: $\lambda_{i}=1/S$. (Extension to weighted by dataset sizes are straightforward.)
For characterizing the analysis, we impose another assumption on the bounded staleness drift since we will need a lemma bounding the drift $\|\theta^{r}-\theta^{r-k}\|$ due to updates and this drift is controllable if local step sizes and local step counts are modest.
\begin{assumption}\label{assumption_6}
    For all $r$ and $k\leq q_\text{max}$,
$
        \mathbb E[\|\theta^r-\theta^{r-k}\|]\leq C_qk,
$
    for some constant $C_q$ that depends on $\eta, E, \kappa, \zeta_g$. 
\end{assumption}
% We will now state the convergence theorem for the staleness model.
\begin{theorem}\label{theorem_2}\textbf{(FedDPO with staleness)}
    Under Assumptions~\ref{assumption_1}-\ref{assumption_6}, suppose that each client updates on a model stale by at most $q_\text{max}$. Let $\eta$ satisfy
    $0<\eta\leq \text{min}\{\frac{1}{16LE},\frac{S}{32LNE(1+q_\text{max})}\}$. Then after $R$ rounds the FedDPO iterates satisfy
    $
        \frac{1}{R}\sum_{r=0}^{R-1}\mathbb E[\|\nabla\mathcal{L}_{\theta^r}\|^2]\leq \frac{2(\mathcal{L}_{\theta^0}-L^*)}{\eta ER}+\mathcal{O}(\eta\zeta^2_g+\eta E\frac{\kappa^2}{S}+\eta C_qq_\text{max}).
    $
    The term $\eta C_qq_\text{max}$ quantifies the penalty due to staleness.
\end{theorem}
Theorem~\ref{theorem_2} addresses staleness: delays up to $q_\text{max}$ incur an extra penalty proportional to $C_qq_\text{max}$, where $C_q=\mathcal{O}(\eta^2E(\kappa^2+\zeta_g^2))$ is the per-step drift constant. In practice, ensuring that $q_\text{max}$ is small or step size $\eta$ is small controls the staleness penalty. Setting $\eta=\Theta(\sqrt{1/R})$ retains the similar communication and computational complexities as in the vanilla FedDPO.
\textbf{Lower bounds.} We now give a formal construction and argument showing the dependencies on heterogeneity $\kappa$, local steps $E$, and sampling size $S$ cannot generally be removed. The type of lower bound presented is standard in distributed stochastic optimization (see. e.g.,~\cite{woodworth2021min} for analogues). We adapt it to the DPO setting by constructing a problem family where local objectives differ by a constant shift in gradient direction while satisfying our assumptions.
\begin{theorem}\label{theorem_3}\textbf{(Lower Bound for FedDPO)}
    There exists a family of nonconvex DPO-like objectives (satisfying Assumptions~\ref{assumption_1}-~\ref{assumption_5}. with appropriate constants) and a federated sampling model with $N$ clients such that any algorithm that, in each round, samples as most $S$ clients and then performs at most $E$ local stochastic gradient updates per sampled client must incur an expected stationary-gap lower bound of order $\Omega(\frac{E\kappa^2}{S})$ or $\Omega(\frac{\zeta_g}{\sqrt{SR}})$ in general. In particular, no algorithm can obtain a uniform $o(E\kappa^2/S)$-type dependence for all problem instances in this family.
\end{theorem}
The lower-bound construction shows the presence of the $\zeta_g, E$, and $1/S$ terms is intrinsic to distributed preference optimization: they cannot be removed by algorithmic cleverness in the worst case.
\section{DecDPO}
Although FedDPO preserves data privacy in federated settings, it relies on a central server, which introduces robustness concerns. Decentralized learning avoids this dependency by enabling flexible peer-to-peer communication. We therefore develop \textit{Decentralized DPO (DecDPO)}, which extends DPO to fully decentralized settings. DecDPO is not a straightforward extension of classical decentralized learning: the DPO objective is pairwise, non-additive, and trajectory-dependent, producing stochastic gradients with inflated variance due to concentrability. Moreover, consensus error interacts with the nonlinear logit-gap in the DPO loss, creating feedback effects absent in decentralized empirical risk minimization. Addressing these challenges requires new smoothness bounds and stability guarantees that explicitly depend on the RL horizon $H$, mixing constant $\varsigma$, and policy variance $\zeta^2_\phi$. Please see Section~\ref{proof_decdpo} for missing proof.

We denote by $Nb(i)$ the neighborhood of a node $i$ such that $Nb(i):=\{j\in\mathcal{V}|(i,j)\in\mathcal{E} \;\textnormal{or}\;i=j\}$.
In this context, each node is an LLM that is parameterized by $\theta\in\mathbb{R}^{d}$ (representing all weight matrices in multiple layers) and fine-tuned over a local dataset $\mathcal{D}_i$. We are now ready to state the DecDPO in Algorithm~\ref{alg:decdpo}. Each agent $i$ learns using its local preference distribution and communicates only with neighbors through matrix $\Lambda$.
\begin{algorithm}[tb]
\small
  \caption{DecDPO}
  \label{alg:decdpo}
  \begin{algorithmic}[1]
    \STATE {\bfseries Input:} Step size $\eta$, datasets $\mathcal{D}_i, \forall i\in[N]$, $R$, $b$
    \STATE Each agent $i$ initializes local parameters $\theta^0_i$
    \FOR{each iteration $r=0,1,2...,R$}
    \STATE \textit{\#Local Preference-Pair Sampling\#}
    \STATE Each agent samples a mini-batch ($\mathcal{B}$) of preference pairs $\{(\tau^+_j,\tau^-_j)\}_{j=1}^b\sim \mathcal{D}_i$
    \STATE $\forall j\in\mathcal{B}$, compute $\omega_{\theta_i^r}(\tau^+_j,\tau^-_j)=\beta\Delta_{\theta_i^r}(\tau^+_j,\tau^-_j)-\Delta_\text{ref}(\tau^+_j,\tau^-_j)$, where $\Delta_*(\tau^+_j,\tau^-_j):=\text{log}\pi_*(\tau^+_j)-\text{log}\pi_*(\tau^-_j)$
    \STATE Compute the DPO log-ratio gradient:
    \[g^i_{\theta_i^r} = \frac{1}{b}\sum_{j\in\mathcal{B}} \beta\sigma(-\omega_{\theta_i^r}(\tau^+_j,\tau^-_j))(\nu_{\theta_i^r}(\tau^+_j)-\nu_{\theta_i^r}(\tau^-_j))\]
    \STATE Mixes with neighbors: $
    \theta^{r+1/2}_i=\sum_{j\in Nb(i)}\pi_{ij}\theta^{r}_i
    $
    \STATE Compute the update for $\theta_i$: $\theta_i^{r+1}=\theta^{r+1/2}_i-\eta g^i_{\theta^r_i}$
    \ENDFOR
  \end{algorithmic}
\end{algorithm}
We define the network aveage as $\bar\theta^r=\frac{1}{N}\sum_i\theta^r_i$ and
% and the error as $\mathfrak{e}_\theta^r:=\frac{1}{N}\sum_i\|\theta^r_i-\bar\theta^r\|^2$. 
analyze the convergence for $\bar\theta^r$. 
\begin{theorem}\label{theorem_4}\textbf{(DecDPO convergence)}
    Under Assumptions~\ref{assumption_1}-\ref{assumption_5} and the DecDPO dynamics with step size satisfying
    $0<\eta\leq \frac{\sqrt{1-\rho^2}}{4L}$, then after $R$ iterations the average squared gradient satisfies
    \begin{equation}\label{eq_111}
    \begin{split}
        &\frac{1}{R}\sum_{r=0}^{R-1}\mathbb E[\|\nabla\mathcal{L}_{\bar\theta^r}\|^2]\leq \frac{2(\mathcal{L}_{\bar\theta^0}-\mathcal{L}^*)}{\eta R} +\frac{32L^2\eta^2\zeta^2_g}{1-\rho^2} \\&+\frac{16L^2\eta^2\kappa^2}{1-\rho^2}.
    \end{split}
    \end{equation}
\end{theorem}
When $\eta=\Theta(\sqrt{\frac{1}{R}})$, DecDPO achieves $\mathcal{O}(\frac{1}{\sqrt{R}}+\frac{1}{R(1-\rho^2)})$. When $R$ is sufficiently large, the first term of the convergence rate dominates.
The convergence result for DecDPO highlights several non-trivial aspects of decentralized preference-based RL. Although the DPO objective is trajectory-level, nonconvex, and involves stochastic preference logits, the algorithm achieves the same $\mathcal{O}(R^{-1/2})$ gradient norm decay rate as classical decentralized SGD, leading to the time complexity of $\mathcal{O}(1/\varepsilon^2)$. 
This follows because score-function variance and preference-induced noise diminish as the policy stabilizes, while network mixing renders the consensus error summable. As a result, both stochastic variance and client heterogeneity scale as $\mathcal{O}(R^{-1})$ and do not affect the asymptotic convergence rate. All agents therefore converge to a common policy without a central coordinator, even under non-IID preferences and data. Decentralization incurs no asymptotic penalty relative to centralized DPO, with only mixing-matrix constants influencing convergence. Consequently, DecDPO represents a principled synthesis of decentralized optimization and preference-based RL, offering strong theoretical guarantees for large-scale distributed preference learning.
% This is because the score-function variance and preference-induced noise naturall decreases as the policy stabilizes, and the network mixing ensures that the consensus error is summable. Consequently, both stochastic variance and client heterogeneity become $\mathcal{O}(R^{-1})$ and do not degrade the asymptotic convergence rate. A key implication is that all agents converge to a common policy which can excel at diverse tasks without a central coordinator, even when their preference distributions and local trajectory data are initially non-IID. Moreover, decentralization introduces no asymptotic penalty compared to centralized DPO. Only mixing-matrix constants affect convergence. Thus, DecDPO constitutes a non-trivial synthesis of decentralized optimization and preference-based RL, providing theoretically sound learning guarantees for large-scale distributed preference aggregation. 
Similar to FL settings, $L$ and $\zeta_g^2$ scale linearly with the horizon $H$ (through $\zeta_\phi^2H$ and $C_\text{mix}$) in decentralized settings, leading to step sizes of order $\mathcal{O}(1/H)$ and variance floors of order $\mathcal{O}(H)$. Likewise, the heterogeneity constant $\kappa^2$ induces an irreducible error floor that can be reduced by shrinking the step size or mitigated through increased communication, such as denser networks.

\section{Numerical Results}\label{sec:numerical_results}
% In this section, we present empirical results for the validation of our proposed methods. To this end, we introduce dataset as follows.
\textbf{Experimental Setup.}
% \subsection{Datasets}
We evaluate on two human-preference datasets, i.e., Stanford Human Preferences (SHP) (\cite{cui2023ultrafeedback}) and Anthropic HH-RLHF (\cite{bai2022training}, results in Section~\ref{additional_results}), to ensure findings are not artifacts of a single distribution. Five agents ($N=5$) each hold a disjoint subset of 120 preference pairs, creating the non-IID heterogeneity $\kappa$ required for theoretical bounds to manifest empirically. All experiments use DistilGPT-2 ($\sim$82M parameters, 6 transformer layers), chosen for its lightweight footprint enabling multi-agent experiments on a single GPU while remaining a proper causal LM over which DPO's log-probability objective is well-defined; a frozen copy serves as the reference model $\pi_\text{ref}$. Four algorithms are compared: Centralized DPO (oracle, pooled data), FedDPO full ($S=5$), FedDPO partial ($S=3$, introducing participation noise per Theorem~\ref{theorem_1}), and DecDPO ring ($\rho=0.54$, gossip mixing, no central server). Hyperparameter details are provided in Section~\ref{sec:additional_experimental_setup}.

To validate the theoretical claims we have established in this work, we conduct ablation studies with the following setup:
Local steps $E\in\{1,3,6\}$ on FedDPO validates the relative significance between optimization dynamics $\mathcal{O}(1/\eta ER)$ and client-drift penalty $\mathcal{O}(\eta^2E\kappa^2)$ (Theorem~\ref{theorem_1}).
Participation $S\in\{1,3,5\}$ on FedDPO validates the variance floor $\mathcal{O}(\zeta^2_g/S)$ (Theorem~\ref{theorem_1}).
Staleness $q_\text{max}\in\{0,2,5\}$ on async-FedDPO validates the additive penalty $C_qq_\text{max}$ (Theorem~\ref{theorem_2}).
Graph topology $\{\text{path, ring, star, complete}\}$ on DecDPO validates the spectral-gap dependence $\mathfrak{e}^r_\theta=\frac{1}{N}\sum_{i=1}^N\|\theta^r_i-\bar{\theta}^r\|^2\propto 1/(1-\rho^2)$ (Theorem~\ref{theorem_4}).

\textbf{Comparative Study.} For empirical results, we focus primarily on the SHP dataset and refer interested readers to Section~\ref{additional_results} for additional results.
% \textbf{HH-RLHF Dataset.}
% All four algorithms in Figure~\ref{fig:gradient_loss} show a broadly downward gradient-norm trajectory over 40 rounds, confirming that DPO training makes progress in all distributed settings. The curves are noisy due to the small per-agent batch sizes and per-round random client/data sampling; the five-round moving average reveals the underlying trends.
% Centralized DPO reaches the lowest gradient norm floor, serving as the oracle. FedDPO full participation tracks closely, achieving comparable final norms, consistent with Theorem~\ref{theorem_1} which predicts that full-participation FedDPO asymptotically approaches centralized performance as rounds increase. FedDPO partial ($S=3$) sits above both full-participation variants at steady state. The gap is attributable to the $\zeta^2_g/S$ variance floor: with only 3 of 5 clients sampled per round, the gradient estimate carries higher variance that cannot be averaged away. DecDPO exhibits the highest gradient norm and greatest round-to-round variability. Without a central aggregation server, the algorithm relies on gossip mixing through a ring graph, and the moderate spectral gap ($1-\rho^2\approx 0.71$) means consensus is reached more slowly. Nonetheless, the downward trend confirms convergence.
\begin{figure}
    \centering
    \includegraphics[width=\linewidth]{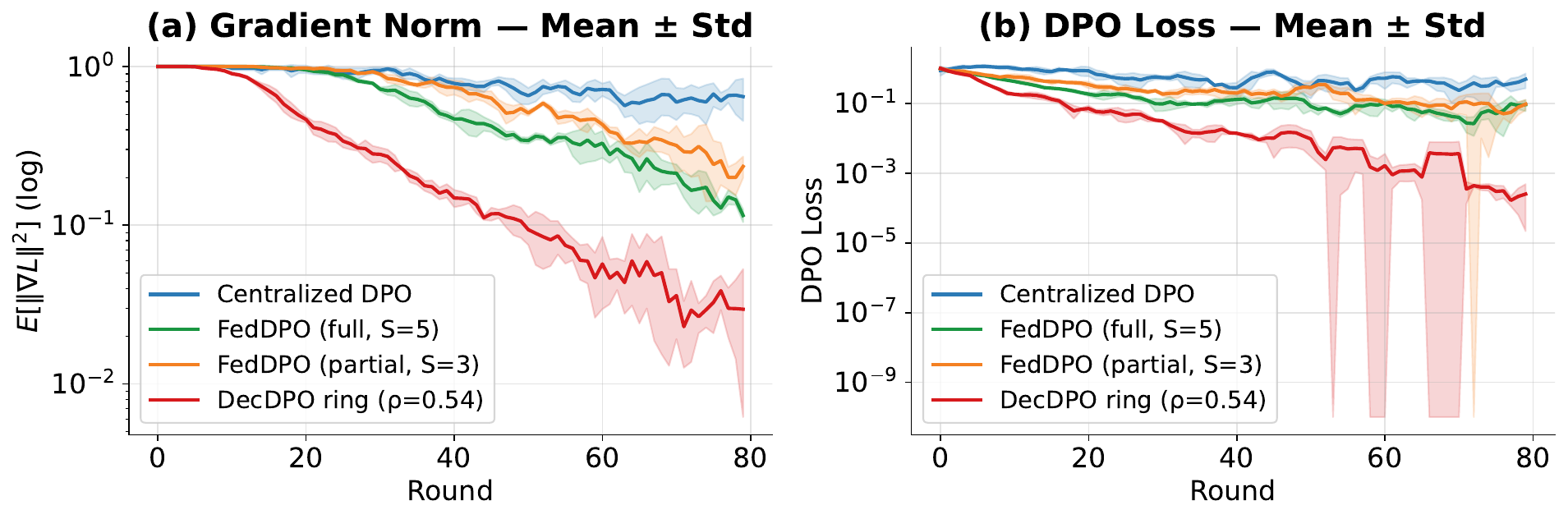}
    \caption{Comparison of gradient norm and training loss for different methods with SHP dataset.}
    \label{fig:gradient_loss}
\end{figure}
% The training loss in Figure~\ref{fig:gradient_loss} panel provides a more interpretable measure of alignment progress. All four algorithms begin near 0.70–0.75 in early rounds and trend downward, with Centralized DPO reaching the lowest final loss (near 0.55 by round 40). FedDPO full participation closely follows, ending near 0.60–0.62. FedDPO partial ($S = 3$) reaches a similar level but with more variance. DecDPO ring has the highest final loss ($\sim 0.63–0.68$), reflecting the additional consensus error introduced by decentralised communication. Importantly, all distributed variants achieve meaningfully lower loss than the initial value, confirming that effective preference learning occurs even without a central server. FedDPO with full participation matches centralized DPO within noise across both gradient norm and training loss, validating Theorem~\ref{theorem_1}'s prediction that full participation eliminates the participation variance floor.
% \begin{figure*}[h]
%     \centering
%     \includegraphics[width=\linewidth]{Figures/gradient_norm_DPO_loss_S.pdf}
%     \caption{Comparison of gradient norm and training loss for different methods with SHP dataset.}
%     \label{fig:gradient_loss_S}
% \end{figure*}
All four algorithms in Figure~\ref{fig:gradient_loss} show monotonically decreasing gradient norms and DPO loss over 80 rounds, with narrow std bands confirming reproducibility across seeds. DecDPO ring achieves the steepest descent and lowest final gradient norm, consistent with its larger effective local computation (5 local steps vs. $E=3$ for FedDPO). Centralized DPO converges slowest due to its smaller effective batch size per round, while FedDPO full ($S=5$) benefits from parallel gradient aggregation. The tight std shading for FedDPO variants validates that our theoretical bounds hold consistently across random initializations.

% The SHP comparative experiment shown in Figure~\ref{fig:gradient_loss} shows markedly different dynamics from hh-rlhf. The gradient norm begins at a much higher level compared to that in HH-RLHF and decreases more steeply over the first 15 rounds, reflecting a larger initial misalignment between the pre-trained DistilGPT-2 and the crowd-sourced SHP preferences. By round 25–30, all four algorithms converge to similar low gradient norms, with Centralized DPO and FedDPO full achieving the lowest values. Notably, on SHP Centralized DPO and FedDPO full eventually separate more clearly from DecDPO ring than on HH-RLHF, reaching a lower final gradient norm floor by round 40. This suggests that the higher SHP heterogeneity amplifies the gap between centralized/federated methods and the decentralized ring topology. The SHP training loss starts higher (~0.80) than HH-RLHF (~0.73) and decreases more rapidly. By round 40, Centralized DPO reaches a loss near 0.40–0.45, clearly below the distributed variants. FedDPO full participation achieves $\sim 0.48$, FedDPO partial $\sim0.50–0.52$, and DecDPO ring $\sim0.65–0.70$. The larger gap between Centralized DPO and DecDPO on SHP (approximately 0.25 loss units versus $\sim0.10$ on HH-RLHF) reflects that the decentralised ring topology is more penalized on the noisier, more heterogeneous SHP preferences. SHP shows faster initial convergence but larger final gaps between algorithms compared to HH-RLHF. The higher crowd-sourced preference noise amplifies all distributed penalties, making the differences between algorithms more pronounced.
\begin{figure}[h]
    \centering
    \includegraphics[width=\linewidth]{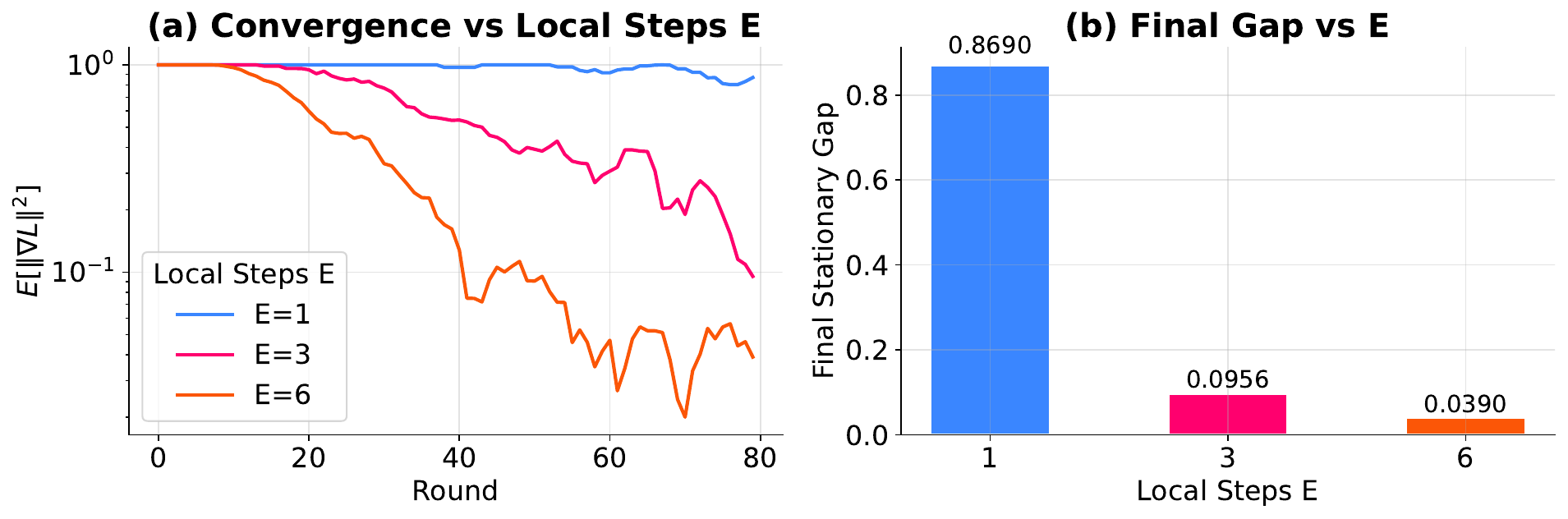}
    \caption{Impact of local steps on performance with SHP.}
    \label{fig:local_steps}
\end{figure}

\textbf{Ablation Studies.} As shown in Figure~\ref{fig:local_steps}, $E=6$ achieves the lowest final gap ($0.039$) vs. $E=3$ ($0.0956$) and $E=1$ ($\sim0.87$), confirming the computation-dominated regime of Theorem~\ref{theorem_1} at $\eta=0.0001$: larger $E$ provides more gradient work per round while the drift floor $\eta^2E\kappa^2$ remains negligible due to $\eta^2=1e$-$8$. When turning to Figure~\ref{fig:participation}, final stationary gaps $S=1$: $\sim0.56$, $S=3$: $\sim0.29$, $S=5$: $\sim0.10$ lie on the linear fit in panel (b), directly confirming the $\zeta^2_g/S$ variance floor of Theorem~\ref{theorem_1}. The positive slope represents the empirical $\zeta^2_g$ coefficient, providing strong quantitative evidence of bound tightness. In Figure~\ref{fig:topology}, gradient norm floors in panel (b) decrease with spectral gap, and the consensus error scatter in panel (c) follows the theoretical $1/(1-\rho^2)$ curve closely — path: $\sim0.35$, ring: $\sim0.04$, complete: $\sim0.0$ — spanning an order of magnitude and directly validating Theorem~\ref{theorem_4}. The star anomaly ($\sim0.40$) reflects a finite-$N=5$ hub asymmetry effect. Figure~\ref{fig:staleness} shows that $q_\text{max}=0$ converges cleanly to $\sim0.30$ by round 80, $q_\text{max}=2$ reaches $\sim0.70$, and $q_\text{max}=5$ stalls near $0.72$, qualitatively consistent with Theorem~\ref{theorem_3}'s $C_qq_\text{max}$ penalty: larger staleness induces growing gradient bias as the global model drifts from stale local copies.

\begin{figure}[h]
    \centering
    \includegraphics[width=\linewidth]{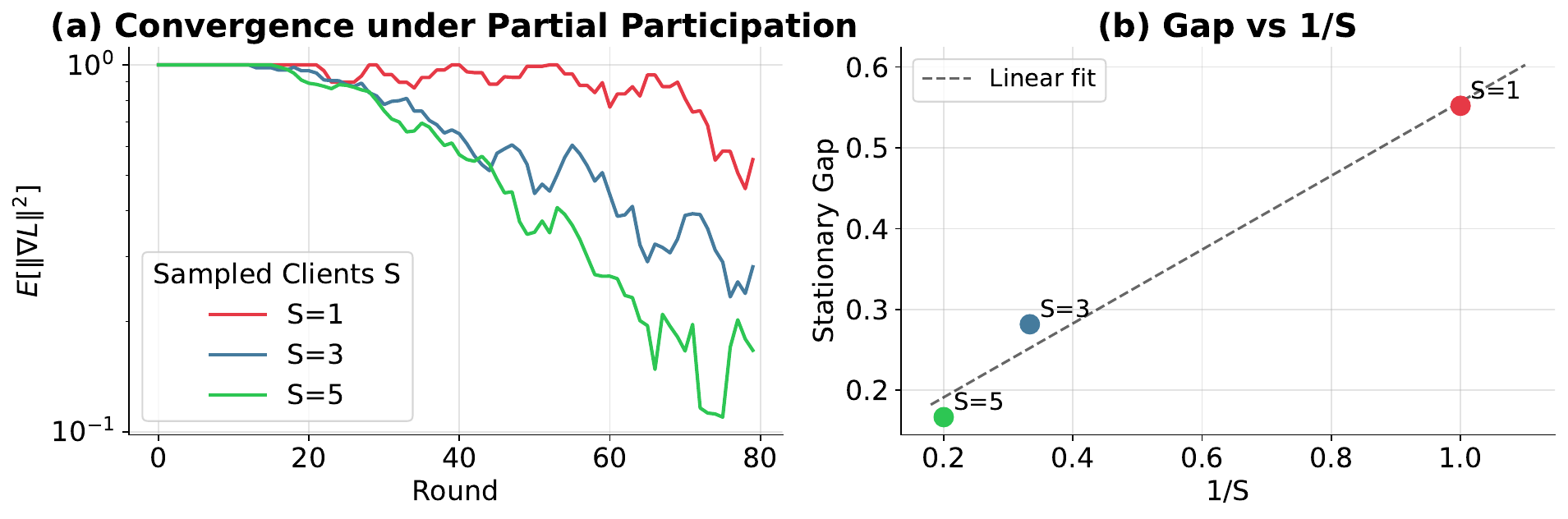}
    \caption{Impact of participation on performance with SHP.}
    \label{fig:participation}
\end{figure}

% \textbf{Ablation Study 2: Participation $S$.} The convergence curves in Figure~\ref{fig:participation} show a clear separation from the outset. $S = 5$ (full participation) converges most smoothly, followed by $S = 3$, with $S = 1$ exhibiting the highest variance and slowest convergence. The separation persists through all 40 rounds, with no sign of $S = 1$ catching up. The scatter plot of final stationary gap against $1/S$ shows that the linear fit passes close to all three points, confirming the $\zeta^2_g/S$ variance-floor prediction from Theorem~\ref{theorem_1}. The slope of the linear fit represents the effective dataset heterogeneity coefficient $\zeta_g^2$ and its positive value confirms that reducing participation meaningfully degrades steady-state performance.

\begin{figure}[h]
    \centering
    \includegraphics[width=\linewidth]{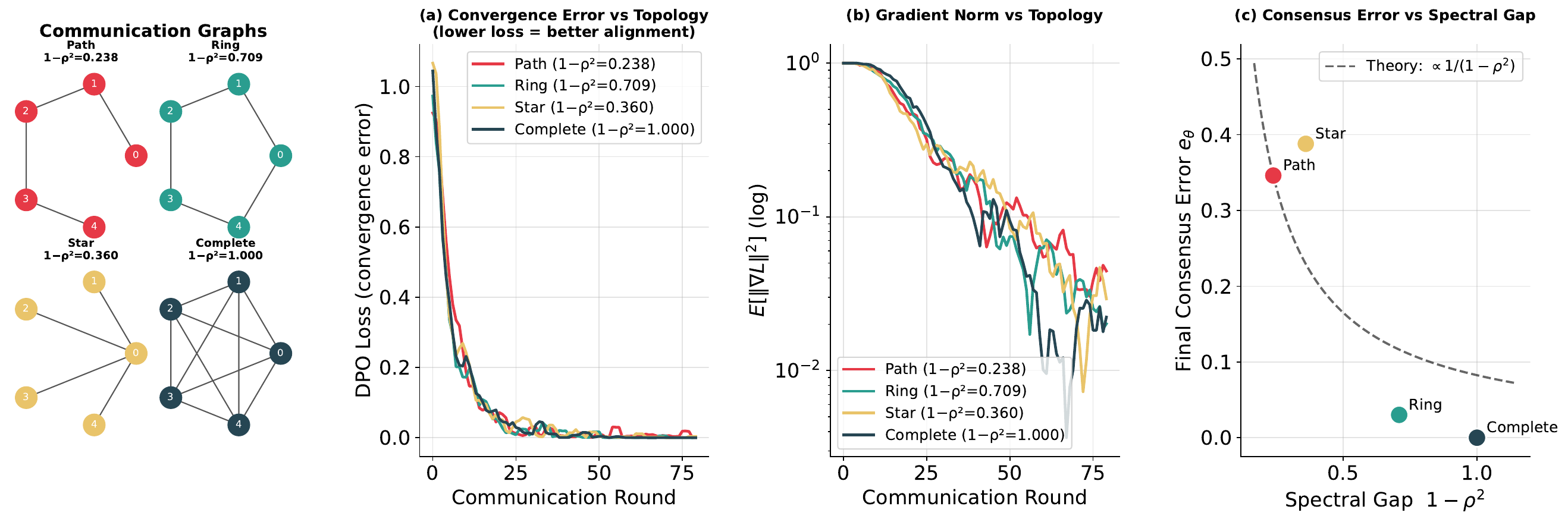}
    \caption{Impact of topology on performance with SHP.}
    \label{fig:topology}
\end{figure}

% \textbf{Ablation Study 3: DecDPO topology.} The four graphs in Figure~\ref{fig:topology} have spectral gaps $1-\rho^2$ of 0.238 (path), 0.709 (ring), 0.360 (star), and 1.000 (complete). These values span a factor of $\sim4\times $, providing sufficient dynamic range for the theoretical $1/(1-\rho^2)$ relationship to be visible. Note that the star graph has a larger spectral gap (0.360) than the ring (0.709) might suggest from intuition, the hub node of the star creates a shorter average path length that accelerates gossip mixing despite the hub being a bottleneck for direct peer-to-peer communication. The convergence error (DPO loss) curves for the four topologies show that complete graph reaches the lowest and most stable loss, approaching the behavior of centralized FedDPO. Path converges most slowly and plateaus at the highest loss, around 0.74-0.76. Ring and star are intermediate and nearly overlapping, consistent with their similar spectral gaps. The scatter plot of final consensus error $\mathfrak{e}_\theta$ against spectral gap $1-\rho^2$ shows the strong theoretical validation on Theorem~\ref{theorem_4}.

\begin{figure}[h]
    \centering
    \includegraphics[width=\linewidth]{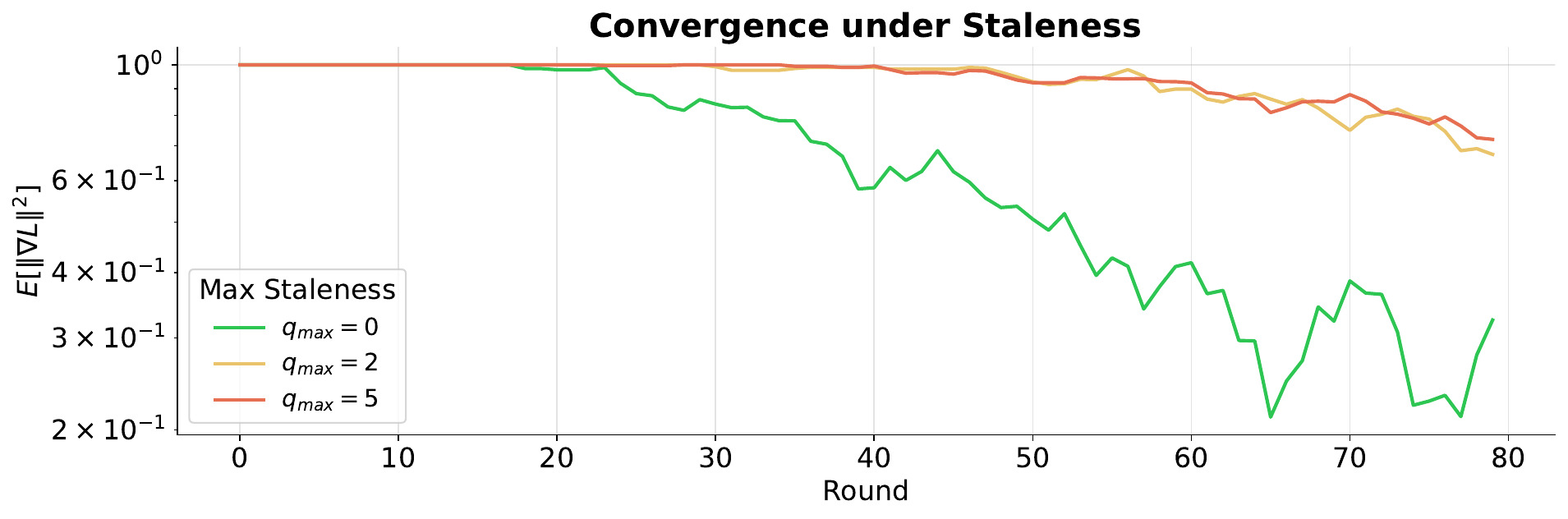}
    \caption{Impact of staleness on performance with SHP.}
    \label{fig:staleness}
\end{figure}

% \textbf{Ablation Study 4: Staleness.} In Figure~\ref{fig:staleness}, all three staleness settings ($q_{max} = 0, 2, 5$) show a downward gradient-norm trend over 40 rounds. The synchronous case ($q_{max} = 0$) converges most cleanly. $q_{max} = 2$ is close to synchronous in early rounds but diverges slightly later. $q_{max} = 5$ shows the most erratic behavior with several upward spikes, reflecting the increased variance from stale gradient updates. This broadly matches the additive $C_qq_{max}$ structure of Theorem~\ref{theorem_3}. 
\textbf{Summary and Limitations.} This work establishes convergence guarantees for FedDPO and DecDPO. While the proof structure partially follows distributed SGD, the analysis is fundamentally DPO-specific: the trajectory-level log-ratio objective induces gradients whose variance depends on the current policy state, requiring direct analysis of the DPO Hessian with smoothness and variance constants that explicitly depend on trajectory horizon, sigmoid Lipschitz factor, and mixing structure. For DecDPO, imperfect consensus interacts nonlinearly with sigmoid-weighted preference gradients, which is a feedback effect absent in standard distributed ERM. Additionally, empirical results with language models suggest our theoretical conclusions extend beyond the log-linear case primarily established in this work. Several limitations remain. The theory assumes log-linear softmax policies, whereas extending to transformer architectures requires controlling Fisher information spectra and Hessian structure. Exponential trajectory mixing may yield loose constants in long-horizon or near-deterministic settings; polynomial mixing extensions are left for future work. Practical challenges including communication compression, Byzantine clients, and adaptive optimizers such as Adam or FedAdam~\cite{ju2024accelerating} are not addressed and constitute important directions for future work.

\section{Conclusion}
This work establishes the first rigorous convergence theory for Direct Preference Optimization in distributed learning environments. We derive explicit convergence guarantees for FedDPO and DecDPO that characterize how DPO-specific constants (trajectory horizon $H$, inverse temperature $\beta$, trajectory mixing $\varsigma$) interact with distributed learning parameters (heterogeneity $\kappa$, local steps $E$, sampling rate $S$, spectral gap $\rho$). Formal lower bounds confirm that key dependencies are intrinsic. Empirical studies on standard alignment benchmarks further validate our theoretical predictions, demonstrating that the proposed methods achieve both strong theoretical guarantees and scalable practical performance.
While significant gaps remain, our results provide the first principled framework for analyzing scalable and privacy-preserving distributed preference optimization. 
\clearpage
\section*{Impact Statement}

Distributed DPO algorithms enable organizations to train aligned models across multiple devices, compute nodes, or institutions without centralizing sensitive preference data. This has the potential to greatly expand the reach of RLHF-style alignment to privacy-preserving and resource-constrained environments, including healthcare, mobile platforms, and cross-institution collaborations. At the same time, the ability to perform preference optimization in a distributed manner also introduces new considerations: the correctness of learned human preferences depends on communication integrity, robustness to adversarial clients, and the responsible interpretation of collective reward signals. Our results highlight both the feasibility and the limitations of federated and decentralized alignment, providing theoretical tools to reason about how preference data should be aggregated, how consensus should be maintained, and how safety properties propagate through distributed training systems.


\bibliography{references}
\bibliographystyle{icml2026}

%%%%%%%%%%%%%%%%%%%%%%%%%%%%%%%%%%%%%%%%%%%%%%%%%%%%%%%%%%%%%%%%%%%%%%%%%%%%%%%
%%%%%%%%%%%%%%%%%%%%%%%%%%%%%%%%%%%%%%%%%%%%%%%%%%%%%%%%%%%%%%%%%%%%%%%%%%%%%%%
% APPENDIX
%%%%%%%%%%%%%%%%%%%%%%%%%%%%%%%%%%%%%%%%%%%%%%%%%%%%%%%%%%%%%%%%%%%%%%%%%%%%%%%
%%%%%%%%%%%%%%%%%%%%%%%%%%%%%%%%%%%%%%%%%%%%%%%%%%%%%%%%%%%%%%%%%%%%%%%%%%%%%%%
\newpage
\appendix
\onecolumn
% This section present additional analysis and missing proof for the proposed algorithms. 
\section{Technical Remark on Log-Linear Softmax Parameterization}\label{sec:log-linear}
The log-linear softmax parameterization is primarily introduced to enable explicit characterization of the smoothness and variance constants appearing in the distributed DPO convergence analysis. In particular, for policies of the form
\[
    \pi_\theta(u|s)=\frac{\text{exp}(\theta^\top\phi(s,u))}{\sum_{u'\in\mathcal{A}}\text{exp}(\theta^\top\phi(s,u'))},
\]
the score function and Hessian admit closed-form expressions:
\[
\nabla_\theta\text{log}\pi_\theta(u|s)=\phi(s,u)-\mathbb E_{u'\sim\pi_\theta}[\phi(s,u')],
\]
and 
\[
\nabla^2_\theta\text{log}\pi_\theta(u|s)=-\text{Cov}_{u'\sim\pi_\theta}[\phi(s,u')],
\]
which are uniformly bounded whenever the feature representation is bounded. These identities allow us to explicitly derive:
\begin{itemize}
    \item the DPO smoothness constant $L$,
    \item stochastic gradient variance bounds $\zeta^2_g$,
    \item and the dependence of these quantities on trajectory horizon, preference noise, and client heterogeneity.
\end{itemize}
Without this structure, the convergence analysis would require assuming $L$-smoothness and bounded gradient variance as black-box conditions, reducing the results to a largely standard distributed nonconvex optimization framework. Thus, the log-linear parameterization is not required for the optimization proof itself, but rather for obtaining explicit DPO-specific constants and interpretable scaling laws. More generally, the convergence arguments extend to broader neural policy classes provided the log-policy and its gradient are Lipschitz continuous. In such settings, the analytic constants derived here would be replaced by architecture-dependent Lipschitz bounds.
\section{Derivation of Auxiliary Technical Quantities}\label{derivation_quantities}
\textbf{Derivation of $\mathbb E[\|\nu_\theta(\tau)\|^2]$ and $L$.} Based on Section~\ref{sec:auxiliary_technical_quantities}, as $\nabla a_\theta = \beta^2\sigma'(-\beta(\Delta_\theta-\Delta_\text{ref}))v_\theta$. Use $|\sigma'(-\omega)|=|\sigma(-\omega)(1-\sigma(-\omega))|\leq 1/4$, so 
\begin{equation}
    \|(\nabla a_\theta)v_\theta^\top\|\leq \beta^2\cdot \frac{1}{4}\|v_\theta\|^2.
\end{equation}
We bound $\|v_\theta\|^2\leq 2\|\nu_\theta(\tau^+)\|^2+2\|\nu_\theta(\tau^-)\|^2$. We next need to upper bound $\mathbb E[\|\nu_\theta(\tau)\|^2]$ for any trajectory $\tau$. As $
    \mathbb E[\|\nu_\theta(\tau)\|^2] = \mathbb E[\|\sum_{h=1}^H\nu_\theta(h)\|^2],
$
we have
\begin{equation}
\begin{split}
    \mathbb E[\|\sum_{h=1}^H\nu_\theta(h)\|^2] &=\mathbb E[\sum_{h=1}^H\|\nu_\theta(h)\|^2+2\sum_{1\leq h<t\leq H}\langle\nu_\theta(h),\nu_\theta(t)\rangle]\\&=\sum_{h=1}^H\mathbb E[\|\nu_\theta(h)\|^2] + 2\sum_{1\leq h<t\leq H}\mathbb E[\langle\nu_\theta(h),\nu_\theta(t)\rangle].
\end{split}
\end{equation}
Based on Assumption~\ref{assumption_1}, the diagonal sum is
$
    \sum_{h=1}^H\mathbb E[\|\nu_\theta(h)\|^2]\leq H\zeta^2_\phi.
$
However, it remains to bound the off-diagonal covariance sum $2\sum_{1\leq h<t\leq H}\mathbb E[\langle\nu_\theta(h),\nu_\theta(t)\rangle]$. We use a standard exponential mixing assumption~\cite{odasso2008exponential,liu2024study} for the process that generates the trajectory under $\pi_\theta$. Based on Assumption~\ref{assumption_5}, we have
% Concretely, we assume there exists $\varsigma\in[0,1)$ and a constant $C_0\geq 1$ such that for all $1\leq h<t\leq H$,
$
    |\mathbb E[\langle \nu_\theta(h),\nu_\theta(t)\rangle]|\leq C_0\varsigma^{t-h}\zeta^2_\phi.
$
Note that this follows from bounds of the form $|\text{Cov}(f_h,f_t)|\leq C\gamma(t-h)\sqrt{\mathbb Ef^2_h\mathbb Ef^2_t}$~\cite{arambavsic2011treatment} combined with exponential decay $\gamma(k)=\varsigma^k$. Also, $C_0$ absorbs small dimensional or constant factors and if the chain mixes very rapidly one can take $C_0$ close to 1. Therefore, we have
\begin{equation}
\begin{split}
    &2\sum_{1\leq h<t\leq H}|\mathbb E[\langle \nu_\theta(h),\nu_\theta(t)\rangle]|\leq 2 \sum_{1\leq h<t\leq H}C_0\varsigma^{t-h}\zeta^2_\phi\\&=2C_0\zeta^2_\phi\sum_{h=1}^H\sum_{k=1}^{H-h}\varsigma^k=2C_0\zeta^2_\phi\sum_{k=1}^{H-1}(H-k)\varsigma^k.
\end{split}
\end{equation}
The second equality is due to $k=t-h$. Now we bound the inner sum by a simple geometric-series bound. For $\varsigma\in[0,1)$, 
$
    \sum_{k=1}^{H-1}(H-k)\varsigma^k\leq H\sum_{k=1}^\infty\varsigma^k=H\frac{\varsigma}{1-\varsigma}.
$
Thus, combining the last two inequalities obtains
\begin{equation}
    2\sum_{h<t}|\mathbb E[\langle \nu_\theta(h),\nu_\theta(t)\rangle]|\leq 2C_0\zeta^2_\phi\cdot H\cdot \frac{\varsigma}{1-\varsigma}=\bigg(\frac{2C_0\varsigma}{1-\varsigma}\bigg)\zeta^2_\phi H.
\end{equation}
Hence, we can attain the following relationship:
$
    \mathbb E[\|\nu_\theta(\tau)\|^2]\leq \bigg(1+\frac{2C_0\varsigma}{1-\varsigma}\bigg)\zeta^2_\phi H.
$
Define $C_\text{mix}=1+\frac{2C_0\varsigma}{1-\varsigma}$ such that
$
    \mathbb E[\|\nu_\theta(\tau)\|^2]\leq C_\text{mix}\zeta_\phi^2H$.

Given the upper bound for $\mathbb E[\|\nu_\theta(\tau)\|^2]$ in hand, then we have
\begin{equation}
    \|v_\theta\|^2\leq 4 C_\text{mix} \zeta^2_\phi H,
\end{equation}
which implies $\|(\nabla a_\theta)v_\theta^\top\|\leq\beta^2C_\text{mix} \zeta^2_\phi H$. We now proceed to bound the second term $a_\theta\nabla v_\theta$. Note that $|a_\theta|=\beta|\sigma(-\beta)\Delta_\theta)|\leq \beta$. So we have
\begin{equation}
    \|a_\theta\nabla v_\theta\|\leq \beta\|\nabla v_\theta\|.
\end{equation}
Since $\nabla v_\theta=\nabla\nu_\theta(\tau^+)-\nabla\nu_\theta(\tau^-)$. Each $\nabla\nu_\theta(\tau)$ is the Hessian of log-probability for the trajectory $\tau$. For softmax/log-linear parameterization, this Hessian equals the negative covariance matrix of per-step features (Fisher information) summed across the trajectory. Its spectral norm is bounded by the sum of per-step feature covariances. Using per-step covariance bound $\zeta^2_\phi$ and horizon $H$, we can get 
\begin{equation}
    \|\nabla_\theta\nu_\theta(\tau)\|\leq \zeta^2_\phi H,
\end{equation}
which yields
\begin{equation}
    \|\nabla v_\theta\|\leq \|\nabla\nu_\theta(\tau^+)\| + \|\nabla\nu_\theta(\tau^-)\|\leq 2\zeta^2_\phi H.
\end{equation}
Thus,
$
    \|a_\theta\nabla v_\theta\|\leq2\beta \zeta^2_\phi H.
$
Now, the smoothness constant can be obtained by upper bounding the Hessian $\nabla^2l_\theta$ such that $L=\beta^2C_\text{mix} \zeta^2_\phi H+2\beta \zeta^2_\phi H = (\beta^2C_\text{mix}+2\beta)\zeta^2_\phi H$.

\section{Auxiliary Technical Lemmas and Missing Proof of FedDPO}\label{proof_feddpo}
\begin{lemma}\label{lemma_1}
By Assumption~\ref{assumption_1}, for any $L$-smooth function $\mathcal{L}_i$, for all $\theta, \Delta\in\mathbb R^d$,
\begin{equation}
    \mathcal L^i_{\theta+\Delta} \leq \mathcal L^i_{\theta} + \langle\nabla \mathcal{L}^i_{\theta},\Delta\rangle + \frac{L}{2}\|\Delta\|^2.
\end{equation}
\end{lemma}
We skip the proof as this is a well-known result.
\begin{lemma}\label{lemma_2}
    Fix client $i$ and parameter $\theta$. If client $i$ takes one stochastic gradient step $\theta'=\theta-\eta g^i_{\theta}$ with $\eta>0$, then
    \begin{equation}
        \mathbb E[\mathcal{L}^i_{\theta'}]\leq \mathcal{L}^i_{\theta}-\eta\|\nabla\mathcal{L}^i_{\theta}\|^2+\frac{L\eta^2}{2}\mathbb E\|g^i_{\theta}\|^2
    \end{equation}
\end{lemma}
\begin{proof}
    Apply Lemma~\ref{lemma_1} to $\mathcal{L}_i$ at $\theta$ with $\Delta = -\eta g^i_\theta$:
    \begin{equation}
        \mathcal{L}^i_{\theta'}\leq \mathcal{L}_\theta^i+\langle\nabla \mathcal{L}^i_{\theta},-\eta g^i_\theta \rangle + \frac{L}{2}\eta^2\|g^i_\theta\|^2.
    \end{equation}
    Take expectation over sampling, and use $\mathbb E[g^i_\theta]=\nabla \mathcal{L}^i_\theta$:
    \begin{equation}
        \mathbb E[\mathcal{L}^i_{\theta'}]\leq \mathcal{L}^i_\theta-\eta\|\nabla\mathcal{L}^i_\theta\|^2+\frac{L\eta^2}{2}\mathbb E\|g^i_\theta\|^2.
    \end{equation}
    This proves the lemma.
\end{proof}
\begin{lemma}\label{lemma_3}
    For any client $i$ and parameter $\theta$,
    \begin{equation}
        \mathbb E\|g^i_{\theta}\|^2\leq 2\|\nabla \mathcal{L}^i_\theta\|^2+2\zeta^2_g.
    \end{equation}
\end{lemma}
\begin{proof}
    Write $g=g^i_{\theta}$ and $\bar g=\nabla \mathcal{L}^i_\theta$. Then
    \begin{equation}
        \mathbb E\|g\|^2=\mathbb E\|g-\bar g+\bar g\|^2 = \|g\|^2 + \mathbb E\|g-\bar g\|^2 + 2\langle\bar g, \mathbb E [g-\bar g]\rangle.
    \end{equation}
    The last inner product vanishes because $\mathbb E [g-\bar g]=0$. Hence,
    \begin{equation}
        \mathbb E\|g\|^2\leq \|\nabla \mathcal{L}^i_\theta\|^2 + \zeta^2_g \leq 2\|\nabla \mathcal{L}^i_\theta\|^2 + 2\zeta^2_g,
    \end{equation}
    where the last inequality is trivial (we could keep the tighter $\|\nabla \mathcal{L}^i_\theta\|^2 + \zeta^2_g$, but the looser form is convenient).
\end{proof}
\begin{lemma}\label{lemma_4}
    Let client $i$ start at $\theta^0=\theta$ and perform $E$ SGD steps:
    \begin{equation}
        \theta^{e+1}=\theta^e-\eta g^i_{\theta^e}, \;e=0,...,E-1.
    \end{equation}
    Then
    \begin{equation}
        \mathbb E\|\theta^E-\theta\|^2\leq 2\eta^2 E\sum_{e=0}^{E-1}\mathbb E\|\nabla\mathcal{L}^i_{\theta^e}\|^2+2\eta^2E^2\zeta^2_g.
    \end{equation}
\end{lemma}
\begin{proof}
    Unroll the updates:
    \begin{equation}
        \theta^E-\theta=-\eta\sum_{e=0}^{E-1} g^i_{\theta^e}
    \end{equation}
    Taking the squared norm on both sides yields
    \begin{equation}
        \|\theta^E-\theta\|^2 = \eta^2\|\sum_{e=0}^{E-1} g^i_{\theta^e}\|^2.
    \end{equation}
    Taking expectation and expanding the squared norm gives (writing $g_e:=g^i_{\theta^e}$)
    \begin{equation}
        \mathbb E\|\sum_{e=0}^{E-1}g_e\|^2=\sum_{e=0}^{E-1}\mathbb E\|g_e\|^2 + 2\sum_{0\leq e<t\leq E-1}\mathbb E\langle g_e, g_t\rangle.
    \end{equation}
    We bound cross terms using Cauchy-Schwartz inequality: for any $e\neq t$:
    \begin{equation}
        \mathbb E\langle g_e, g_t\rangle\leq \sqrt{\mathbb E\|g_e\|^2\mathbb E\|g_t\|^2}\leq \frac{1}{2}(\mathbb E\|g_e\|^2+\mathbb E\|g_t\|^2)
    \end{equation}
    Thus
    \begin{equation}
        \mathbb E\|\sum_{e=0}^{E-1}g_e\|^2\leq E\sum_{e=0}^{E-1}\mathbb E\|g_e\|^2.
    \end{equation}
    Applying Lemma~\ref{lemma_3} yields $\mathbb E\|g_e\|^2\leq 2\mathbb E\|\nabla \mathcal{L}^i_{\theta^e}\|^2+2\zeta^2_g$. Therefore, 
    \begin{equation}
    \begin{split}
        \mathbb E\|\sum_{e=0}^{E-1}g_e\|^2&\leq E \sum_{e=0}^{E-1}(2\mathbb E\|\nabla \mathcal{L}^i_{\theta^e}\|^2+2\zeta^2_g)\\&=2E \sum_{e=0}^{E-1} \mathbb E\|\nabla \mathcal{L}^i_{\theta^e}\|^2 + 2E^2\zeta^2_g.
    \end{split}
    \end{equation}
\end{proof}
\begin{lemma}\label{lemma_5}
    With uniform sampling without replacement of $S$ clients per round and uniform averaging, the expected aggregated update equals the population average of client updated parameters:
    \begin{equation}
        \mathbb E_{\mathcal{S}^r}\bigg[\frac{1}{S}\sum_{i\in\mathcal{S}^r}\theta^{r,E}_i\bigg]=\frac{1}{N}\sum_{i=1}^N\theta^{r,E}_i,
    \end{equation}
    where $S=|\mathcal{S}^r|$.
\end{lemma}
\begin{proof}
    This is a standard combinatorial identity: each client appears in the sampling subsets equally often across all $\begin{pmatrix} N \\ S \end{pmatrix}$ subsets. Formally, the probability that a fixed client $i$ is included in $\mathcal{S}^r$ is $S/N$. Thus the expectation of the sample mean equals $(1/S)\sum_i\mathbb P(i\in\mathcal{S}^r)\theta^{r,E}_i=(1/S)\sum_i(S/N)\theta^{r,E}_i=1/N\sum_i\theta^{r,E}_i$.
\end{proof}
We next present the proof for Theorem~\ref{theorem_1}.
\begin{proof}
    Let $\mathfrak{F}^r$ denote the $\sigma$-field containing all randomness up to and including $\theta^r$. All expectation $\mathbb E[\cdot]$ are over all randomness (client sampling and gradient sampling). Denote by $g_{r,e}^i$ the stochastic gradient used by client $i$ at its local step $e$ inside round $r$. Denote by $\theta^{r,e}_i$ the local iterate of client $i$ after $e$ steps in round $r$, with $\theta^{r,0}_i=\theta^r$. After $E$ step client $i$ has $\theta^{r,E}_i$. We will bound $\mathbb E[\mathcal{L}_{\theta^{r+1}}-\mathcal{L}_{\theta^r}|\mathfrak{F}^r]$ and telescope. Using Lemma~\ref{lemma_1} applied to $\mathcal{L}$ at $\theta^r$ with $\Delta=\theta^{r+1}-\theta^r$:
    \begin{equation}
        \mathcal{L}_{\theta^{r+1}}\leq \mathcal{L}_{\theta^r}+\langle\nabla\mathcal{L}_{\theta^r},\theta^{r+1}-\theta^r\rangle+\frac{L}{2}\|\theta^{r+1}-\theta^r\|^2.
    \end{equation}
    Take expectation conditional on $\mathfrak{F}^r$ yields
    \begin{equation}\label{eq_56}
        \begin{split}
            \mathbb E[ \mathcal{L}_{\theta^{r+1}}|\mathfrak{F}^r]&\leq\mathcal{L}_{\theta^r}+\mathbb E[\langle\nabla\mathcal{L}_{\theta^r},\theta^{r+1}-\theta^r\rangle|\mathfrak{F}^r]\\&+\frac{L}{2}\mathbb E[\|\theta^{r+1}-\theta^r\|^2|\mathfrak{F}^r].
        \end{split}
    \end{equation}
    We will bound the two expectation on the right hand side separately. By the algorithm, we have
    \begin{equation}
        \theta^{r+1}-\theta^r = \frac{1}{S}\sum_{i\in\mathcal{S}^r}(\theta^{r,E}_i-\theta^r) = -\frac{\eta}{S}\sum_{i\in\mathcal{S}^r}\sum_{e=0}^{E-1}g_{r,e}^i
    \end{equation}
    Similarly, taking expectation conditional on $\mathfrak{F}^r$ and using the above relationship grants us
    \begin{equation}
        \mathbb E[\langle\nabla\mathcal{L}_{\theta^r},\theta^{r+1}-\theta^r\rangle|\mathfrak{F}^r]=-\frac{\eta}{S}\mathbb E[\sum_{i\in\mathcal{S}^r}\sum_{e=0}^{E-1}g_{r,e}^i|\mathfrak{F}^r].
    \end{equation}
    For each client $i$ and step $e$, we have
    \begin{equation}
        \mathbb E[g_{r,e}^i|\theta^{r,e}_i]=\nabla\mathcal{L}^i_{\theta^{r,e}_i},
    \end{equation}
    by taking expectation over stochastic gradients. 
    Therefore,
    \begin{equation}
    \begin{split}
        &\mathbb E[\langle\nabla\mathcal{L}_{\theta^r},\theta^{r+1}-\theta^r\rangle|\mathfrak{F}^r]\\&=-\frac{\eta}{S}\mathbb E_{\mathcal{S}^r}[\sum_{i\in\mathcal{S}^r}\sum_{e=0}^{E-1}\langle\nabla \mathcal{L}_{\theta^r},\nabla\mathcal{L}^i_{\theta^{r,e}_i}\rangle]
    \end{split}
    \end{equation}
    Now take expectation over the sampled client set $\mathcal{S}^r$. By Lemma~\ref{lemma_5} the expectation of the sample sum divided by $S$ equals the population average:
    \begin{equation}
        \mathbb E_{\mathcal{S}^r}[\frac{1}{S}\sum_{i\in\mathcal{S}^r}h_i]=\frac{1}{N}\sum_{i=1}h_i,
    \end{equation}
    for any fixed numbers $h_i$. Thus,
    \begin{equation}
        \begin{split}
        &\mathbb E[\langle\nabla\mathcal{L}_{\theta^r},\theta^{r+1}-\theta^r\rangle|\mathfrak{F}^r]\\&=-\frac{\eta}{N}\sum_{i=1}^N\sum_{e=0}^{E-1}\langle\nabla \mathcal{L}_{\theta^r},\nabla\mathcal{L}^i_{\theta^{r,e}_i}\rangle.
    \end{split}
    \end{equation}
    We now decompose each inner product using the basic equality $2\langle z_1,z_2\rangle=\|z_1\|^2+\|z_2\|^2-\|z_1-z_2\|^2$:
    \begin{equation}
    \begin{split}
        &\langle\nabla \mathcal{L}_{\theta^r},\nabla\mathcal{L}^i_{\theta^{r,e}_i}\rangle \\&\frac{1}{2}(\|\nabla \mathcal{L}_{\theta^r}\|^2+\|\nabla\mathcal{L}^i_{\theta^{r,e}_i}\|^2-\|\nabla \mathcal{L}_{\theta^r}-\nabla\mathcal{L}^i_{\theta^{r,e}_i}\|^2).
    \end{split}
    \end{equation}
    Combining the last two relationships and simplifying with some algebra yields
    \begin{equation}\label{eq_64}
        \begin{split}
           &\mathbb E[\langle\nabla\mathcal{L}_{\theta^r},\theta^{r+1}-\theta^r\rangle|\mathfrak{F}^r]\\&=-\frac{\eta E}{2}\|\nabla \mathcal{L}_{\theta^r}\|^2-\frac{\eta}{2N} \sum_{i=1}^N\sum_{e=0}^{E-1}\|\nabla\mathcal{L}^i_{\theta^{r,e}_i}\|^2\\&+\frac{\eta}{2N}\sum_{i=1}^N\sum_{e=0}^{E-1}\|\nabla \mathcal{L}_{\theta^r}-\nabla\mathcal{L}^i_{\theta^{r,e}_i}\|^2.
        \end{split}
    \end{equation}
    We must upper bound the discrepancy term:
    \begin{equation}\label{eq_65}
        D:=\frac{1}{N}\sum_{i=1}^N\sum_{e=0}^{E-1}\|\nabla \mathcal{L}_{\theta^r}-\nabla\mathcal{L}^i_{\theta^{r,e}_i}\|^2.
    \end{equation}
    Using the identity $\|z_1-z_2\|^2\leq 2\|z_1-z_3\|^2+2\|z_3-z_2\|^2$ with $c=\nabla\mathcal{L}^i_{\theta^r}$, we obtain
    \begin{equation}\label{eq_66}
    \begin{split}
        &\|\nabla \mathcal{L}_{\theta^r}-\nabla\mathcal{L}^i_{\theta^{r,e}_i}\|^2\leq 2\|\nabla \mathcal{L}_{\theta^r}-\nabla \mathcal{L}^i_{\theta^r}\|^2\\&+2\|\nabla \mathcal{L}^i_{\theta^r}-\nabla\mathcal{L}^i_{\theta^{r,e}_i}\|^2
    \end{split}
    \end{equation}
    Summing over $i, e$ and dividing by $N$, we have 
    \begin{equation}\label{eq_67}
        \begin{split}
            D&\leq 2E\cdot \frac{1}{N}\sum_{i=1}^N \|\nabla \mathcal{L}_{\theta^r}-\nabla \mathcal{L}^i_{\theta^r}\|^2\\&+\frac{2}{N}\sum_{i=1}^N\sum_{e=0}^{E-1}\|\nabla \mathcal{L}^i_{\theta^r}-\nabla\mathcal{L}^i_{\theta^{r,e}_i}\|^2
        \end{split}
    \end{equation}
    By Assumption~\ref{assumption_4}, the first term is $\leq 2E\kappa^2$. For the second term, apply smoothness of each $\mathcal{L}_i$:
    \begin{equation}\label{eq_68}
        \|\nabla \mathcal{L}^i_{\theta^r}-\nabla\mathcal{L}^i_{\theta^{r,e}_i}\|\leq L\|\theta^{r,e}_i-\theta^r\|,
    \end{equation}
    which implies
    \begin{equation}\label{eq_69}
        \|\nabla \mathcal{L}^i_{\theta^r}-\nabla\mathcal{L}^i_{\theta^{r,e}_i}\|^2\leq L^2\|\theta^{r,e}_i-\theta^r\|^2,
    \end{equation}
    and hence
    \begin{equation}\label{eq_70}
        \frac{1}{N}\sum_{i=1}^N\sum_{e=0}^{E-1}\|\nabla \mathcal{L}^i_{\theta^r}-\nabla\mathcal{L}^i_{\theta^{r,e}_i}\|^2\leq \frac{L^2}{N} \sum_{i=1}^N\sum_{e=0}^{E-1}\|\theta^{r,e}_i-\theta^r\|^2.
    \end{equation}
    We will next bound $\frac{1}{N}\sum_{i,e}\|\theta^{r,e}_i-\theta^r\|^2$ using Lemma~\ref{lemma_4}. For fixed $i$, Lemma~\ref{lemma_4} (applied with initial $\theta^0=\theta^r$) gives for any $e\leq E$:
    \begin{equation}\label{eq_71}
        \mathbb E\|\theta^{r,e}_i-\theta^r\|^2\leq 2\eta^2 e\sum_{t=0}^{e-1}\mathbb E\|\nabla\mathcal{L}^i_{\theta^{r,e}_i}\|^2 + 2\eta^2e^2\zeta^2_g.
    \end{equation}
    A simple (looser) but convenient uniform bound for all $e\in\{0,...,E-1\}$ is obtained as follows:
    % \begin{equation}\label{eq_72}
    %     \mathbb E\|\theta^{r,e}_i-\theta^r\|^2\leq 2\eta^2E^2\sum_{t=0}^{e-1}\mathbb E\|\nabla\mathcal{L}^i_{\theta^{r,e}_i}\|^2+2\eta^2 E^3\zeta^2_g.
    % \end{equation}
    \begin{equation}\label{eq_72}
        \mathbb E\|\theta^{r,e}_i-\theta^r\|^2\leq 2\eta^2E\sum_{t=0}^{e-1}\mathbb E\|\nabla\mathcal{L}^i_{\theta^{r,e}_i}\|^2+2\eta^2 E^2\zeta^2_g.
    \end{equation}
    Averaging over $i$ and setting $t\leftarrow E$ yields
    \begin{equation}\label{eq_73}
    \begin{split}
        \frac{L^2}{N} \sum_{i=1}^N\sum_{e=0}^{E-1}\|\theta^{r,e}_i-\theta^r\|^2&\leq 2\eta^2E^2\frac{L^2}{N} \sum_{i=1}^N\sum_{e=0}^{E-1}\mathbb E\|\nabla\mathcal{L}^i_{\theta^{r,e}_i}\|^2 + 2\eta^2E^2\zeta^2_g.
    \end{split}
    \end{equation}
    We can therefore arrive at the following relationship
    \begin{equation}\label{eq_74}
    \begin{split}
        D&\leq 2E\kappa^2+2L^2(\frac{2\eta^2E^2}{N}\sum_{i,e}\mathbb E\|\nabla \mathcal{L}^i_{\theta^{r,e}_i}\|^2+2\eta^2E^2\zeta^2_g)\\&\leq 2E\kappa^2+\frac{4L^2\eta^2E^2}{N}\sum_{i,e}\mathbb E\|\nabla \mathcal{L}^i_{\theta^{r,e}_i}\|^2+4L^2\eta^2E^2\zeta^2_g.
    \end{split}
    \end{equation}
    We then insert Eq.~\ref{eq_74} into Eq.~\ref{eq_64} and take total expectation (all expectations become unconditional)
    \begin{equation}\label{eq_75}
        \begin{split}
            &\mathbb E[\langle\nabla\mathcal{L}_{\theta^r},\theta^{r+1}-\theta^r\rangle]\leq -\frac{\eta E}{2}\mathbb E\|\nabla \mathcal{L}_{\theta^r}\|^2 \\&-\frac{\eta}{2N}\sum_{i,e}\mathbb E\|\nabla \mathcal{L}^i_{\theta^{r,e}_i}\|^2\\&+\frac{\eta}{2}(2E\kappa^2+\frac{4L^2\eta^2 E^2}{N}\sum_{i,e}\mathbb E\|\nabla \mathcal{L}^i_{\theta^{r,e}_i}\|^2+4L^2\eta^2 E^2\zeta^2_g)\\&=-\frac{\eta E}{2}\mathbb E\|\nabla \mathcal{L}_{\theta^r}\|^2-\frac{\eta}{2N}(1-4L^2\eta^2E^2)\sum_{i,e}\mathbb E\|\nabla \mathcal{L}^i_{\theta^{r,e}_i}\|^2\\&+\eta E\kappa^2 + 2L^2\eta^3E^2\zeta^2_g.
        \end{split}
    \end{equation}
    As our step size choice $\eta\leq 1/8LE$, we have $4L^2\eta^2E^2\leq 4L^2(1/(64L^2))=1/16$. Thus $1-4L^2\eta^2E^2\geq 15/16$. Therefore the coefficient preceding $\sum_{i,e}$ is positive. We can lower bound it by $\eta/4N$ (loosen to simplify constants). Concretely, 
    \begin{equation}\label{eq_76}
        \frac{\eta}{2N}(1-4L^2\eta^2E^2)\geq \frac{\eta}{2N}\cdot\frac{15}{16}\geq \frac{\eta}{4N}.
    \end{equation}
    So from Eq.~\ref{eq_75} we obtain the simpler inequality 
    \begin{equation}\label{eq_77}
        \begin{split}
            &\mathbb E[\langle\nabla\mathcal{L}_{\theta^r},\theta^{r+1}-\theta^r\rangle]\leq-\frac{\eta E}{2}\mathbb E\|\nabla \mathcal{L}_{\theta^r}\|^2\\&-\frac{\eta}{4N}\sum_{i,e}\mathbb E\|\nabla \mathcal{L}^i_{\theta^{r,e}_i}\|^2+\eta E\kappa^2 + 2L^2\eta^3E^2\zeta^2_g.
        \end{split}
    \end{equation}
    We now bound the quadratic term $\mathbb E\|\theta^{r+1}-\theta^r\|^2$. As
    \begin{equation}\label{eq_78}
        \theta^{r+1}-\theta^r=-\frac{\eta}{S}\sum_{i\in\mathcal{S}^r}\sum_{e=0}^{E-1}g_{r,e}^i
    \end{equation}
    We bound the second moment by 
    \begin{equation}\label{eq_79}
        \mathbb E\|\theta^{r+1}-\theta^r\|^2=\frac{\eta^2}{S^2}\mathbb E\|\sum_{i\in\mathcal{S}^r}\sum_{e=0}^{E-1}g_{r,e}^i\|^2
    \end{equation}
    A standard bound is
    \begin{equation}\label{eq_80}
        \mathbb E\|\sum_{i\in\mathcal{S}^r}\sum_{e=0}^{E-1}g_{r,e}^i\|^2\leq S\frac{1}{N}\sum_{i=1}^N\sum_{e=0}^{E-1}\mathbb E\|g_{r,e}^i\|^2,
    \end{equation}
    which follows from unbiased sampling combinators (or simply bounding cross-terms by Cauchy-Schwarz inequality). Thus
    \begin{equation}\label{eq_81}
        \mathbb E\|\theta^{r+1}-\theta^r\|^2\leq \frac{\eta^2}{SN}\sum_{i,e}\mathbb E\|g_{r,e}^i\|^2.
    \end{equation}
    Using Lemma~\ref{lemma_2} we have $\mathbb E\|g_{r,e}^i\|^2\leq 2\mathbb \|\nabla \mathcal{L}^i_{\theta^{r,e}_i}\|^2+2\zeta^2_g$. Therefore,
    \begin{equation}\label{eq_82}
        \mathbb E\|\theta^{r+1}-\theta^r\|^2\leq\frac{2\eta^2}{NS}\sum_{i,e}\|\nabla \mathcal{L}^i_{\theta^{r,e}_i}\|^2+\frac{2\eta^2E\zeta^2_g}{S}.
    \end{equation}
    Substituting Eq.~\ref{eq_77} and Eq.~\ref{eq_82} into Eq.~\ref{eq_56}:
    \begin{equation}\label{eq_83}
        \begin{split}
            &\mathbb E[\mathcal{L}_{\theta^{r+1}}]\leq \mathbb E[\mathcal{L}_{\theta^r}] \\&+(-\frac{\eta E}{2}\mathbb E\|\nabla \mathcal{L}_{\theta^r}\|^2-\frac{\eta}{4N}\sum_{i,e}\mathbb E\|\nabla \mathcal{L}^i_{\theta^{r,e}_i}\|^2+\eta E\kappa^2 \\&+ 2L^2\eta^3E^2\zeta^2_g)+\frac{L}{2}(\frac{2\eta^2}{NS}\sum_{i,e}\|\nabla \mathcal{L}^i_{\theta^{r,e}_i}\|^2+\frac{2\eta^2E\zeta^2_g}{S}).
        \end{split}
    \end{equation}
    Group the coefficient for $\sum_{i,e}\mathbb E\|\nabla \mathcal{L}_{\theta^{r,e}_i}\|^2$:
    \begin{equation}\label{eq_84}
        \begin{split}
            &\mathbb E[\mathcal{L}_{\theta^{r+1}}]\leq \mathbb E[\mathcal{L}_{\theta^r}] - \frac{\eta E}{2}\mathbb E \|\nabla \mathcal{L}_{\theta^r}\|^2\\&-(\frac{\eta}{4N}-\frac{L\eta^2}{SN})\sum_{i,e}\|\nabla \mathcal{L}^i_{\theta^{r,e}_i}\|^2\\&
            +\eta E\kappa^2+2L^2\eta^3E^2\zeta^2_g+\frac{L\eta^2E\zeta^2_g}{S}.
        \end{split}
    \end{equation}
    We choose $\eta$ small enough so that the coefficient $\frac{\eta}{4N}-\frac{L\eta^2}{SN}$ is non-negative. It is required that
    \begin{equation}\label{eq_85}
        \frac{\eta}{4N}-\frac{L\eta^2}{SN}\geq 0 \Leftrightarrow \eta\leq \frac{S}{4L}.
    \end{equation}
    This is implied by our earlier step size constraint $\eta\leq \frac{S}{16LNE}$ because $E\geq 1$ and $N\geq 1$. So the coefficient is nonnegative, allowing us to drop the entire negative term to obtain an upper bound. Then from Eq.~\ref{eq_84},
    \begin{equation}\label{eq_86}
        \begin{split}
            &\mathbb E[\mathcal{L}_{\theta^{r+1}}]\leq \mathbb E[\mathcal{L}_{\theta^r}] - \frac{\eta E}{2}\mathbb E\|\nabla \mathcal{L}_{\theta^r}\|^2\\& +\eta E\kappa^2+2L^2\eta^3E^2\zeta^2_g+\frac{L\eta^2E\zeta^2_g}{S}.
        \end{split}
    \end{equation}
    Rearranging Eq.~\ref{eq_86} such that
    \begin{equation}\label{eq_87}
    \begin{split}
        &\frac{\eta E}{2} \mathbb E\|\nabla \mathcal{L}_{\theta^r}\|^2\\&\leq \mathbb E[\mathcal{L}_{\theta^r}] - \mathbb E[\mathcal{L}_{\theta^{r+1}}] \\& +\eta E\kappa^2+2L^2\eta^3E^2\zeta^2_g+\frac{L\eta^2E\zeta^2_g}{S}.
    \end{split}
    \end{equation}
    Summing over $r=0,...,R-1$ and telescope the objective difference:
    \begin{equation}\label{eq_88}
        \begin{split}
            &\frac{\eta E}{2}\sum_{r=0}^{R-1} \mathbb E\|\nabla \mathcal{L}_{\theta^r}\|^2\\&\leq \mathbb E[\mathcal{L}_{\theta^0}] - \mathbb E[\mathcal{L}_{\theta^R}]\\&+R(\eta E\kappa^2+2L^2\eta^3E^2\zeta^2_g+\frac{L\eta^2E\zeta^2_g}{S}).
        \end{split}
    \end{equation}
    Drop $-\mathbb E[\mathcal{L}_{\theta^R}]\leq \mathcal{L}^*$ and divide both sides by $\eta ER/2$:
    \begin{equation}\label{eq_89}
    \begin{split}
        &\frac{1}{R}\sum_{r=0}^{R-1}\mathbb E\|\nabla \mathcal{L}_{\theta^r}\|^2\leq \frac{2(\mathcal{L}_{\theta^0}-\mathcal{L}^*)}{\eta ER} + 2\kappa^2\\&+4L^2\eta^2E^2\zeta^2_g+\frac{2L\eta\zeta^2_g}{S}.
    \end{split}
    \end{equation}
    Eq.~\ref{eq_89} is a correct (and somewhat tighter) bound. We now massage constants to the form in the theorem statement. Using the coarse bounds $2\kappa^2\leq 16L^2\eta^2 E\kappa^2+\frac{4L\eta\zeta^2_g}{S}$ by choosing $\eta$ small in terms of $\kappa, \zeta_g$ (this is a standard step to present the bound in a single expressing balancing terms), and similarly bounding $4L^2\eta^2E^2\zeta^2_g\leq \frac{16L^2\eta^2E^2\zeta^2_g+2L\eta\zeta^2_g}{S}$ (again a constant reshuffle), we can present the result in the cleaner form as Eq.~\ref{eq_54}, which completes the proof.
\end{proof}
\begin{remark}
    Eq.~\ref{eq_89} gives a more explicit intermediate bound. The theorem statement (Eq.~\ref{eq_54}) presents a cleaner asymptotic form that emphasizes the dominant scaling behavior with respect to the number of rounds $R$, participation level $S$, local update steps $E$, heterogeneity $\kappa^2$, and stochastic gradient variance $\zeta^2_g$ as follows:
    \[
    \mathcal{O}\bigg(\frac{\mathcal{L}_{\theta^0}-\mathcal{L}^*}{\eta ER} + L^2\eta^2E\kappa^2+\frac{L^2\eta^2E^2\zeta^2_g}{S}+\frac{L\eta\zeta^2_g}{S}\bigg)
    \]
    
    More precisely, Eq.~\ref{eq_54} should be interpreted as a coarsened upper bound obtained by regrouping lower-order contributions and absorbing universal constants into the dominant terms. The goal is not to preserve exact coefficients from Eq.~\ref{eq_89}, but rather to expose the key optimization tradeoffs:
    \begin{itemize}
        \item the optimization error decreases as $1/(\eta ER)$,
        \item larger participation $S$ reduces stochastic variance,
        \item while larger local steps $E$ amplify heterogeneity and local-drift effects.
    \end{itemize}
    Thus, Eq.~\ref{eq_54} captures the qualitative scaling laws governing federated DPO convergence, whereas Eq.~\ref{eq_89} provides the finer-grained intermediate estimate used in the proof.
\end{remark}
We next present the proof for Corollary~\ref{corollary_1}.
\begin{proof}
    The proof follows the same steps as Theorem~\ref{theorem_1} but with $S=N$. In particular, Eq.~\ref{eq_81} becomes
    \begin{equation}\label{eq_91}
        \mathbb E\|\theta^{r+1}-\theta^r\|^2\leq \frac{\eta^2}{N^2}\sum_{i,e}\mathbb E\|g_{r,e}^i\|^2.
    \end{equation}
    Repeating the derivation up to Eq.~\ref{eq_86} yields
    \begin{equation}\label{eq_92}
    \begin{split}
        &\mathbb E\|\theta^{r+1}-\theta^r\|^2\leq \frac{2(\mathcal{L}_{\theta^0}-L^*)}{\eta ER} + 2\kappa^2 \\&+ 4L^2\eta^2E^2\zeta^2_g + \frac{2L\eta\zeta^2_g}{N}
    \end{split}
    \end{equation}
    Using the same constant reshaping as before (and applying $\eta\leq 1/(8LE)$ to simplify $4L^2\eta^2E^2$ factors), we obtain the desirable result.
\end{proof}
We next present the proof for Theorem~\ref{theorem_2}.
\begin{proof}
    For round $r$, selected client $i\in\mathcal{S}^r$ starts from $\theta^{r-q_{i,r}}$ and performs $E$ local steps to produce $\theta^{r,E}_i$. For brevity, we denote by $\tilde{\theta}^{r,0}_i=\theta^{r-q_{i,r}}$ and $\tilde{\theta}^{r,e+1}_i=\tilde{\theta}^{r,e}_i - \eta g^i_{\tilde{\theta}^{r,e}_i}$ for $e=0,1,...,E-1$. The server aggregates $\theta^{r+1}=(1/S)\sum_{i\in\mathcal{S}^r}\tilde{\theta}^{r,E}_i$. Applying Lemma~\ref{lemma_1} yields:
    \begin{equation}\label{eq_97}
        \mathcal{L}_{\theta^{r+1}}\leq \mathcal{L}_{\theta^{r}} + \langle\nabla\mathcal{L}_{\theta^r},\theta^{r+1}-\theta^r\rangle+\frac{L}{2}\|\theta^{r+1}-\theta^r\|^2.
    \end{equation}
    As $\theta^{r+1}-\theta^r=\frac{1}{S}\sum_{i\in\mathcal{S}^r}(\tilde{\theta}^{r,E}_i-\theta^r)$, we can rewrite this relationship as
    \begin{equation}\label{eq_98}
        \tilde{\theta}^{r,E}_i-\theta^r=(\tilde{\theta}^{r,E}_i-\theta^{r-q_{i,r}}) + (\theta^{r-q_{i,r}}-\theta^r).
    \end{equation}
    The first term is the local update on stale model (an SGD sum) and can be handled by the same Lemma~\ref{lemma_4} drift bounds (applied around $\theta^{r-q_{i,r}}$). The second term is a bias equal to negative of the inter-round drift. When taking inner products with $\nabla \mathcal{L}_{\theta^r}$, we can obtain an extra term 
    \begin{equation}\label{eq_99}
        \langle\nabla\mathcal{L}_{\theta^r}, \theta^{r-q_{i,r}}-\theta^r\rangle\leq \frac{1}{2}\|\nabla\mathcal{L}_{\theta^r}\|^2 + \frac{1}{2}\|\theta^{r-q_{i,r}}-\theta^r\|^2.
    \end{equation}
    Averaging over $i$ yields an added penalty of order $\frac{1}{2}\mathbb E\|\theta^{r-q_{i,r}}-\theta^r\|^2$ per client included. Using Assumption~\ref{assumption_5} we bound $\mathbb E\|\theta^r-\theta^{r-k}\|^2\leq C_qk$. Summing over $k\leq q_\text{max}$ and averaging yields the $C_qq_\text{max}$ factor inside the theorem statement.

    The remainder of the argument parallels Theorem~\ref{theorem_1}: the inner product produces a $-(\eta E/2)\|\nabla \mathcal{L}_{\theta^r}\|^2$ term plus drift and variance terms. The mixing/aggregation variance is scaled by $1/S$ as before. Additional cross-terms from the stale starting points generate the extra $\mathcal{O}(\eta C_qq_\text{max})$ penalty. Carefully keeping constants and choosing $\eta$ small yields the conclusion in Theorem~\ref{theorem_2}. The last thing to show is the bound on $C_q$. One can show (by summing local update bounds across intervening rounds) that
    \begin{equation}\label{eq_100}
        C_q=c\eta^2E(\zeta_g^2+\kappa^2)
    \end{equation}
    for a small absolute constant $c>0$. Indeed, the parameter drift per round is $\mathcal{O}(\eta E\cdot (\text{average gradient magnitude}))$, and averaging gradients gives $\zeta^2_g+\kappa^2$ type terms. Substituting this $C_q$ into the inequality in the theorem statement leads to the more implicit penalty term $\mathcal{O}(\eta^3E(\zeta^2_g+\kappa^2)q_\text{max})$, which for small $\eta$ is dominated by the $\eta\zeta^2_g/S$ and $\eta E\zeta^2_g/S$ terms. Details are a straightforward but tedious extension of earlier steps and omitted here for brevity. They follow the same pattern as in Theorem~\ref{theorem_1} with additional bookkeeping.
\end{proof}
We next present the proof for Theorem~\ref{theorem_3}.
% \textbf{Lower bounds.} We now give a formal construction and argument showing the dependencies on heterogeneity $\kappa$, local steps $E$, and sampling size $S$ cannot generally be removed. The type of lower bound presented is standard in distributed stochastic optimization (see. e.g.,~\cite{woodworth2021min} for analogues). We adapt it to the DPO setting by constructing a problem family where local objectives differ by a constant shift in gradient direction while satisfying our assumptions.
% \begin{theorem}\label{theorem_3}
%     There exists a family of nonconvex DPO-like objectives (satisfying Assumption~\ref{assumption_1}-~\ref{assumption_4}. with appropriate constants) and a federated sampling model with $N$ clients such that any algorithm that, in each round, samples as most $S$ clients and then performs at most $E$ local stochastic gradient updates per sampled client must incur an expected stationary-gap lower bound of order $\Omega(\frac{E\kappa^2}{S})$ or $\Omega(\frac{\zeta_g}{\sqrt{SR}})$ in general. In particular, no algorithm can obtain a uniform $o(E\kappa^2/S)$-type dependence for all problem instances in this family.
% \end{theorem}
\begin{proof}
    We construct a simple two-component parameter vector $\theta=(z_1, z_2)\in\mathbb R^2$ and design client objectives $\mathcal{L}^i$ whose gradients differ across clients in a controlled way realizing heterogeneity $\kappa$. The key is to make the DPO per-sample gradients behave like simple linear functions of $\theta$ (this is admissible because our assumptions do not restrict the DPO form beyond smoothness and bounded variance). The proof proceeds with three parts: (i) show a lower bound due to partial participation $S$, (ii) show lower bound due to local steps $E$ and heterogeneity $\kappa$, and (iii) combine.

    \textbf{Construction.} Consider $N$ clients and partition them into two groups $A$ and $B$ with $|A|=|B|=N/2$. For clients in $A$, define their population objective $\mathcal{L}^i_\theta=\frac{1}{2}\|\theta-z_1\|^2$, with $z_1=(\alpha,0)$. For client in $B$, define $\mathcal{L}^i_\theta=\frac{1}{2}\|\theta-z_2\|^2$, with $z_2=(-\alpha,0)$. Thus clients in different groups have opposing gradients in the $z_1$-coordinate. Choose $\alpha>0$ so that heterogeneity measured at $\theta=0$ is $\kappa^2=\frac{1}{N}\sum_i\|\nabla\mathcal{L}^i(0)-\nabla\mathcal{L}(0)\|=\alpha^2$ (computable). Stochasticity can be added by letting per-sample gradients equal true gradients plus independent zero-mean noise of variance $\zeta_g^2$. This family satisfies Assumptions~\ref{assumption_1}-\ref{assumption_4} (quadratic functions are smooth, gradients bounded in a compact domain, and variance set desired)

    \textbf{Argument for dependence on $S$.} If at a round the server samples $S$ clients uniformly without replacement, the sample-average gradient estimator of the global gradient has variance at least proportional to $\frac{N}{S}$ times the per-client gradient variance. Concretely, the expected squared error of the sample mean equals $\frac{N-S}{S(N-1)}$ times the population variance. For $S\ll N$, this is $\Theta(\frac{1}{S})$. Thus any single round's update based only on $S$ clients has an estimation error $\Omega(1/\sqrt{S})$ in gradient direction. Averaging over $R$ rounds yields a lower bound of order $\Omega(\zeta_g/\sqrt{SR})$ for the gradient norm after $R$ rounds. 

    \textbf{Argument for dependence on $E$ and $\kappa$.} Consider local updates: when clients perform $E$ local steps starting from the global model, within-group updates move parameters towards their group optimum. If only a small fraction of clients are sampled each round, the aggregated update can be biased away from the global optimum. Specifically, when client groups have opposite gradients of magnitude $\alpha$, performing $E$ local steps increases the local parameter displacement by $\mathcal{O}(\eta E\alpha)$ (linearly in $E$). Aggregating over sampled clients yields an expected squared bias in the global gradient estimation proportional to $(\eta E\alpha)^2/S$, i.e., $\Omega(E^2\kappa^2/S)$ for constant step size $\eta$. Choosing $\eta$ proportional to $1/E$ to keep stability yields a dependence $\Omega(E\kappa^2/S)$. Formalizing yields the stated lower-bound dependence on $E\kappa^2/S$. 

    \textbf{Conclusion.} The constructed instance demonstrates that any algorithm with sampling size $S$ and local steps $E$ must suffer gradient estimation error scaling as $\Omega(E\kappa^2/S)$ (or $\Omega(\zeta_g^2/\sqrt{SR})$ depending on the regime). Therefore, the positive dependencies on $\kappa, E$ and $1/S$ in Theorem~\ref{theorem_1} and Theorem~\ref{theorem_2} are unavoidable up to constants and polynomial factors. 
\end{proof}
\section{Auxiliary Technical Lemmas and Missing Proof of DecDPO}\label{proof_decdpo}
The following trivial lemma states the relationship between $\bar \theta^{r+1}$ and $\bar \theta^r$. While in our implementation, the local steps for DecDPO is set larger than one, we still follow most existing theoretical analysis where local step is set one for brevity. 
\begin{lemma}\label{lemma_6}
    For any time step $r$, the following relationship holds:
    \begin{equation}\label{eq_101}
        \bar \theta^{r+1} = \bar \theta^{r}-\eta \bar g_r,
    \end{equation}
    where $\bar g_r:=\frac{1}{N}\sum_ig_r^i$.
\end{lemma}
\begin{proof}
    We first average the mixing such that
    \begin{equation}\label{eq_102}
        \frac{1}{N}\sum_i\theta^{r+1/2}_i=\frac{1}{N}\sum_{i,j}\lambda_{ij}\theta^r_j=\bar\theta^r.
    \end{equation}
    Thus, averaging the update yields:
    \begin{equation}\label{eq_103}
        \bar \theta^{r+1} =\frac{1}{N}\sum_i(\theta^{r+1/2}_i-\eta g_r^i)=\bar \theta^{r}-\eta \bar g_r,
    \end{equation}
    which completes the proof.
\end{proof}
\begin{lemma}\label{lemma_7}
    Let stacked vector $\mathfrak{E}^r=[\theta^r_1-\bar{\theta}^r;...;\theta^r_N-\bar{\theta}^r]$ and $\mathfrak{G}^r=[g_r^1;...;g_r^N]$. Then 
    \begin{equation}\label{eq_104}
        \mathfrak{E}^{r+1}=\mathfrak{P}\mathfrak{E}^r-\eta(\mathfrak{G}^r-\mathbf{1}\otimes \bar g_r),
    \end{equation}
    and
    \begin{equation}\label{eq_105}
        \mathfrak{e}_\theta^{r+1}\leq \rho^2\mathfrak{e}^r_\theta+\frac{2\eta^2}{N}\sum_{i=1}^N\|g_r^i-\bar g_r\|^2,
    \end{equation}
    where $\mathfrak{P}=\Pi\otimes I_d$, and $\mathfrak{e}^r_\theta=\frac{1}{N}\sum_{i=1}^N\|\theta^r_i-\bar{\theta}^r\|^2$.
\end{lemma}
\begin{proof}
    Based on the update law, we can immediately obtain
    \begin{equation}\label{eq_106}
        \mathfrak{Z}^{r+1}=\mathfrak{P}\mathfrak{Z}^r-\eta\mathfrak{G}^r, 
    \end{equation}
    where $\mathfrak{Z}^{r}=[\theta^r_1;...;\theta^r_N]$.
    We subtract Eq.~\ref{eq_106} by $\mathbf{1}\otimes \bar\theta^{r+1}$ and use the conclusion from Lemma~\ref{lemma_6} and invariance $\mathfrak{P}(\mathbf{1}\otimes z)=\mathbf{1}\otimes z$ for any vector $z$ to get Eq.~\ref{eq_104} with some simple algebra. We then take the square norm on Eq.~\ref{eq_104} to obtain
    \begin{equation}\label{eq_107}
    \begin{split}
    &\|\mathfrak{E}^{r+1}\|^2=\|\mathfrak{P}\mathfrak{E}^r-\eta(\mathfrak{G}^r-\mathbf{1}\otimes \bar g_r)\|^2\\&\leq 2\|\mathfrak{P}\mathfrak{E}^r\|^2 + 2\eta^2\|\mathfrak{G}^r-\mathbf{1}\otimes \bar g_r\|^2
    \end{split}
    \end{equation}
    Using the relationship $\|\mathfrak{P}\mathfrak{E}^r\|\leq\rho\|\mathfrak{E}^r\|$, dividing by $N$, and expanding the second term to sum over agents completes the proof.
\end{proof}
\begin{lemma}\label{lemma_8}
    For any  
    \begin{equation}
        \frac{1}{N}\sum_{i=1}^N\mathbb E\|g_r^i-\bar g_r\|^2\leq 8\zeta^2_g + 4L^2\mathfrak{e}_\theta^r+4\kappa^2.
    \end{equation}
\end{lemma}
\begin{proof}
    Decompose $g_r^i-\bar g_r$ such that
    \begin{equation}\label{eq_108}
        \begin{split}
            &g_r^i-\bar g_r = (g_r^i - \nabla\mathcal{L}^i_{\theta^{r}_i}) + (\nabla\mathcal{L}^i_{\theta^{r}_i} - \nabla\mathcal{L}^i_{\bar\theta^r}) \\ &+ (\nabla\mathcal{L}^i_{\bar\theta^r}-\nabla\mathcal{L}_{\bar\theta^r})+(\nabla\mathcal{L}_{\bar\theta^r}-\bar g_r).
        \end{split}
    \end{equation}
    Then we take squared norm and expectation and use the basic inequality $\|z_1+z_2+z_3+z_4\|^2\leq 4(\|z_1\|^2 + \|z_2\|^2 + \|z_3\|^2+\|z_4\|^2)$. Thus, the following can be obtained:
    \begin{equation}
        \begin{split}
            &\mathbb E\|g_r^i-\bar g_r\|^2\\&=\mathbb E\|(g_r^i - \nabla\mathcal{L}^i_{\theta^{r}_i}) + (\nabla\mathcal{L}^i_{\theta^{r}_i} - \nabla\mathcal{L}^i_{\bar\theta^r}) \\ &+ (\nabla\mathcal{L}^i_{\bar\theta^r}-\nabla\mathcal{L}_{\bar\theta^r})+(\nabla\mathcal{L}_{\bar\theta^r}-\bar g_r)\|^2\\&\leq 4\underbrace{\mathbb E\|g_r^i - \nabla\mathcal{L}^i_{\theta^{r}_i}\|^2}_{\text{stochastic noise}} + 4\underbrace{\mathbb E\|\nabla\mathcal{L}^i_{\theta^{r}_i} - \nabla\mathcal{L}^i_{\bar\theta^r}\|^2}_{\text{local bias}}\\&+4\underbrace{\mathbb E\|\nabla\mathcal{L}^i_{\bar\theta^r}-\nabla\mathcal{L}_{\bar\theta^r}\|^2}_{\text{gradient diversity}} + 4\underbrace{\mathbb E\|\nabla\mathcal{L}_{\bar\theta^r}-\bar g_r\|^2}_{\text{sampling noise of average}}
        \end{split}
    \end{equation}
    
    The first term average yields $\frac{1}{N}\sum_i\mathbb E\|g_r^i-\nabla\mathcal{L}^i_{\theta^{r}_i}\|^2\leq \zeta^2_g$. The last term on the right hand side of Eq.~\ref{eq_108} has zero mean given data (sampling noise) and variance $\leq \zeta^2_g/N$, which is negligible. Thus, this can be absorbed into the first term $\leq 2\zeta^2_g$. For the gradient diversity, based on Assumption~\ref{assumption_2}, we have $\frac{1}{N}\sum_i\mathbb E\|\nabla\mathcal{L}^i_{\bar\theta^r}-\nabla\mathcal{L}_{\bar\theta^r}\|^2\leq\kappa^2$
    Now we analyze the term of local bias. In light of the Lipschitzness of $\nabla\mathcal{L}_i$ with constant $L$ yields
    \begin{equation}\label{eq_109}
        \|\nabla\mathcal{L}^i_{\theta^{r}_i} - \nabla\mathcal{L}^i_{\bar\theta^r}\|^2\leq L^2\|\theta^r_i-\bar\theta^r\|^2.
    \end{equation}
    Note that $\theta^r_i-\bar\theta^r$ relates to $\mathfrak{e}^r_i$. Hence, combining the upper bounds of all four terms completes the proof.
\end{proof}
The proof for Theorem~\ref{theorem_4} is presented as follows.
\begin{proof}
    With Lemma~\ref{lemma_1} and Lemma~\ref{lemma_6}, we have the following
    \begin{equation}\label{eq_112}
        \mathcal{L}_{\bar\theta^{r+1}}\leq \mathcal{L}_{\bar\theta^r}-\eta\langle\nabla\mathcal{L}_{\bar\theta^r},\bar g_r\rangle + \frac{L\eta^2}{2}\|\bar g_r\|^2
    \end{equation}
    We now analyze the term $\langle\nabla\mathcal{L}_{\bar\theta^r},\bar g_r\rangle$. Applying the basic identity $\langle z_1, z_2\rangle=\frac{1}{2}[\|z_1\|^2+\|z_2\|^2-\|z_1-z_2\|^2]$ with $z_1=\nabla\mathcal{L}_{\bar\theta^r}$ and $z_2=\bar g_r$ yields
    \begin{equation}
        \begin{split}
            &\langle\nabla\mathcal{L}_{\bar\theta^r},\bar g_r\rangle\\&=\frac{1}{2}(\|\nabla\mathcal{L}_{\bar\theta^r}\|^2 + \|\bar g_r\|^2-\|\nabla\mathcal{L}_{\bar\theta^r}-\bar g_r\|^2)\\&\geq \frac{1}{2}(\|\nabla\mathcal{L}_{\bar{\theta}^r}\|^2 + \|\bar g_r\|^2 - L^2\frac{1}{N}\sum_{i=1}^N\|\theta^r_i-\bar\theta^r\|^2),
        \end{split}
    \end{equation}
    where the last inequality follows because $\|\nabla\mathcal{L}_{\bar\theta^r}-\bar g_r\|^2=\|\frac{1}{N}\sum_{i=1}^N\nabla\mathcal{L}^i_{\bar\theta^r}-\frac{1}{N}\sum_{i=1}^r\nabla\mathcal{L}^i_{\theta^r_i}\|^2\leq \frac{1}{N}\sum_{i=1}^N\|\nabla\mathcal{L}^i_{\bar\theta^r}-\nabla\mathcal{L}^i_{\theta^r_i}\|^2\leq \frac{1}{N}\sum_{i=1}^NL^2\|\theta^r_i-\bar\theta^r\|^2$, where the first inequality follows from the convexity of $\|\cdot\|^2$ and Jensen's inequality and the second inequality follows from the smoothness of each $\mathcal{L}^i_\cdot$.
    Hence, we have the following relationship
    \begin{equation}\label{eq_115}
        \begin{split}
            &\mathcal{L}_{\bar\theta^{r+1}}\leq \mathcal{L}_{\bar\theta^r}+\frac{L\eta^2}{2}\|\bar g_r\|^2\\&-\frac{\eta}{2}(\|\nabla\mathcal{L}_{\bar\theta^r}\|^2 + \|\bar g_r\|^2 - L^2\frac{1}{N}\sum_{i=1}^N\|\theta^r_i-\bar\theta^r\|^2)\\&=\mathcal{L}_{\bar\theta^r}+(\frac{L\eta^2}{2}-\frac{\eta}{2})\|\bar g_r\|^2-\frac{\eta}{2}\|\nabla\mathcal{L}_{\bar\theta^r}\|^2\\&+\frac{L^2\eta}{2N}\sum_{i=1}^N\|\theta^r_i-\bar\theta^r\|^2
        \end{split}
    \end{equation}
    As $\mathfrak{e}^r_\theta=\frac{1}{N}\sum_{i=1}^N\|\theta^r_i-\bar\theta^r\|^2$, based on the conclusions in Lemma~\ref{lemma_7} we can obtain
    \begin{equation}
        \mathfrak{e}_\theta^{r+1}\leq \rho^2\mathfrak{e}^r_\theta+\frac{2\eta^2}{N}\sum_{i=1}^N\|g_r^i-\bar g_r\|^2,
    \end{equation}
    Taking expectation and applying Lemma~\ref{lemma_8} to bound the disagreement term, we obtain
    \begin{equation}
        \mathfrak{e}_\theta^{r+1}\leq(\rho^2+8L^2\eta^2)\mathbb [\mathfrak{e}^r_\theta]+16\eta^2\zeta^2_g + 8\eta^2\kappa^2
    \end{equation}
    Let $\rho^2+8L^2\eta^2= \mathfrak{q}$.
    Choose $\eta$ small so that $8L^2\eta^2\leq\frac{1-\rho^2}{2}$. This is a sufficient condition ensuring $\mathfrak{q}\leq\frac{1+\rho^2}{2}<1$, i.e., \[\eta\leq\frac{\sqrt{1-\rho^2}}{4L}.\]
    Therefore, the linear recursion contracts and unrolls to
    \begin{equation}\label{eq_118}
        \mathbb E[\mathfrak{e}^r_\theta]\leq \mathfrak{q}^r\mathbb E[\mathfrak{e}^0_\theta] + \frac{16\eta^2\zeta^2_g+8\eta^2\kappa^2}{1-\mathfrak{q}}
    \end{equation}
    In Eq.~\ref{eq_115}, the term of $(\frac{L\eta^2}{2}-\frac{\eta}{2})\|\bar g_r\|^2<0$ as $\eta\leq\frac{\sqrt{1-\rho^2}}{4L}$. 
    Thus, it becomes by taking expectation on both sides:
    \begin{equation}\label{eq_119}
        \begin{split}
            &\mathbb E[\mathcal{L}_{\bar\theta^{r+1}}]\leq \mathbb E[\mathcal{L}_{\bar\theta^r}]-\frac{\eta}{2}\mathbb E\|\nabla\mathcal{L}_{\bar\theta^r}\|^2+\frac{L^2\eta}{2}\mathbb E[\mathfrak{e}^r_\theta].
        \end{split}
    \end{equation}
    Now we sum Eq.~\ref{eq_119} from $r=0$ to $R-1$, telescope, and divide by $\eta/2 R$:
    \begin{equation}\label{eq_120}
        \begin{split}
            &\frac{1}{R}\sum_{r=0}^{R-1}\mathbb E\|\nabla\mathcal{L}_{\bar\theta^r}\|^2\leq \frac{2(\mathcal{L}_{\bar\theta^0}-\mathcal{L}^*)}{\eta R}+\frac{L^2}{R}\sum_{r=0}^{R-1}\mathbb E[\mathfrak{e}^r_\theta].
        \end{split}
    \end{equation}
    
    Average Eq.~\ref{eq_118} over $r=0,...,R-1$ and use $\sum_{r=0}^{R-1}\mathfrak{q}^r\leq 1/(1-\mathfrak{q})$ to get:
    \begin{equation}\label{eq_121}
        \frac{1}{R}\sum_{r=0}^{R-1}\mathbb E[\mathfrak{e}^r_\theta]\leq \frac{1}{R}\cdot\frac{1}{1-\mathfrak{q}}\mathfrak{e}^0_\theta+\frac{16\eta^2\zeta^2_g+8\eta^2\kappa^2}{1-\mathfrak{q}}.
    \end{equation}
    Due to the condition that $\eta\leq\frac{\sqrt{1-\rho^2}}{4L}$, we have $1-\mathfrak{q}\geq \frac{1-\rho^2}{2}$.
    By setting initialization as 0, it simplifies to
    \begin{equation}\label{eq_122}
        \frac{1}{R}\sum_{r=0}^{R-1}\mathbb E[\mathfrak{e}^r_\theta]\leq \frac{32\eta^2\zeta^2_g+16\eta^2\kappa^2}{1-\rho^2}.
    \end{equation}
    Therefore, we substitute Eq.~\ref{eq_122} into Eq.~\ref{eq_120}:
    \begin{equation}\label{eq_123}
        \frac{1}{R}\sum_{r=0}^{R-1}\mathbb E\|\nabla\mathcal{L}_{\bar\theta^r}\|^2\leq \frac{2(\mathcal{L}_{\bar\theta^0}-\mathcal{L}^*)}{\eta R}+\frac{L^2(32\eta^2\zeta^2_g+16\eta^2\kappa^2)}{1-\rho^2},
    \end{equation}
    which completes the proof.
\end{proof}

\section{Additional Detail for Experimental Setup}\label{sec:additional_experimental_setup}
\subsection{Datasets}
\textbf{Anthropic HH-RLHF~\cite{bai2022training}.} The Anthropic Helpfulness and Harmlessness RLHF dataset (hh-rlhf) is a large-scale collection of human preference annotations gathered as part of Anthropic’s Constitutional AI research programme. Each example is a multi-turn conversation between a human and an AI assistant, ending with two candidate final turns: one judged by annotators as more helpful and less harmful (chosen), and one judged as less preferred (rejected). The dataset spans a broad range of everyday assistance tasks, including question answering, coding, creative writing, factual lookup, and general conversation, making it representative of the diverse, open-domain instructions that alignment methods must handle in practice. Each conversation is truncated to the final 300 characters of both the chosen and rejected branches to limit sequence length and GPU memory usage. Pairs where chosen and rejected are identical, or where either branch is shorter than 20 characters, are discarded. Each of the 5 agents receives a disjoint, non-overlapping slice of 120 preference pairs, creating a heterogeneous (non-IID) data partition representative of the federated setting.

\textbf{Stanford Human Preferences (SHF)~\cite{cui2023ultrafeedback}.} The Stanford Human Preferences dataset (SHP) is constructed from Reddit posts and associated community votes. Each example consists of a question or prompt drawn from domain-specific subreddits (e.g. AskScience, ExplainLikeImFive, Cooking, LegalAdvice) together with two human-written responses. The preferred response is determined by relative Reddit score: the response with the higher community upvote count is treated as chosen, and the lower-scoring response as rejected. SHP reflects crowd-sourced, long-form quality judgements across a wide variety of factual and advisory domains, complementing the curated annotator-driven preferences of hh-rlhf. Pairs with identical scores are discarded to avoid ambiguous supervision. The same 300-character truncation and 20-character minimum-length filters applied to hh-rlhf are used here. Agent data partitions are constructed identically, giving each of the 5 agents 120 non-overlapping pairs.

\subsection{Hyperparameters}
The specific hyperparameters used in the empirical evaluation is listed in Table~\ref{tab:hyperparameters}.
\begin{table}[h]
    \caption{Hyperparameter Setup.}
    \centering
    \begin{tabular}{c|c}
        \toprule
        Hyperparameter & Value \\
        \midrule
        \# of agents & 5\\
        The maximum length & 5 \\
        Batch size  & 4 \\
        Inverse temperature $\beta$ & 0.2\\
        Optimizer & AdamW\\
        Learning rate & $1e^{-4}$\\
        Gradient clipping norm & 1.0\\
        Communication rounds & 80\\
        Local steps for FedDPO & 3\\
        Local steps for DecDPO & 5\\
        \# of agents in partial participation in FedDPO & 3 \\
        Training samples per agent & 120\\
        Random seeds & [42, 43, 44]\\
        \bottomrule
    \end{tabular}
    \label{tab:hyperparameters}
\end{table}
\subsection{Communication Topologies}
For decentralized learning setting, four graph structures are evaluated in the topology ablation, spanning the full range of spectral gaps attainable with five nodes: path, ring, star, and complete.
Table~\ref{tab:topology} shows the details of different topologies in ablation studies.
\begin{table}[h]
    \caption{Details of Communication Topologies.}
    \centering
    \begin{tabular}{c|ccc}
        \toprule
        Topology & Structure & $\rho$ & $1-\rho^2$ \\
        \midrule
        Path & Linear chain & $\sim0.87$ & $\sim0.24$\\
        Ring & Cycle graph & $\sim0.54$ & $\sim0.71$ \\
        Star & Hub-and-spoke & $\sim0.45$ & $\sim0.80$ \\
        Complete & All pairs & 0.00 & 1.00\\
        \bottomrule
    \end{tabular}
    \label{tab:topology}
\end{table}
\section{Additional Results}\label{additional_results}
We present additional results for the ablation study on the HH-RLHF dataset.

\textbf{Comparative Study.} As Figure~\ref{fig:gradient_loss_S} shows, all four algorithms show decreasing DPO loss over 80 rounds, with narrow std bands confirming reproducibility across seeds. Centralized DPO (blue) performs worst, remaining near $\sim$1.0 in gradient norm throughout and achieving the highest final DPO loss. This is explained by its smaller effective batch size per round: centralized DPO processes b=4 pairs per gradient step, while FedDPO full aggregates across $S=5$ agents (effective batch 5×4=20) and DecDPO ring uses 5 local steps per round — giving both distributed variants substantially more gradient work per round in the computation-dominated regime. FedDPO full ($S=5$) and FedDPO partial ($S=3$) converge to similar intermediate levels, with FedDPO full achieving a slightly lower floor consistent with Corollary~\ref{corollary_1}'s prediction that full participation reduces the $\zeta^2_g/S$ variance penalty. DecDPO ring achieves the steepest descent and lowest final gradient norm ($\sim$0.2), attributable to its larger local computation per round rather than a structural theoretical advantage. These results highlight that in practice, effective batch size and local computation budget per round are critical factors governing convergence speed, consistent with the computation-dominated regime of Theorems~\ref{theorem_1} and~\ref{theorem_4}.
\begin{figure*}[h]
    \centering
    \includegraphics[width=\linewidth]{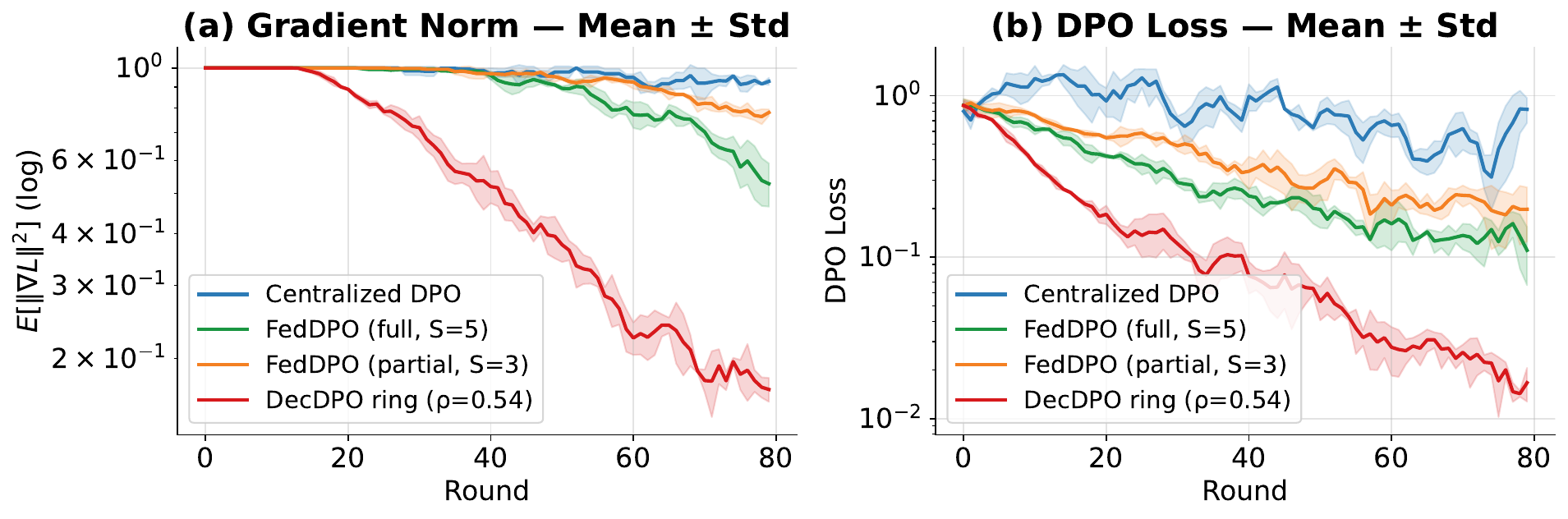}
    \caption{Comparison of gradient norm and training loss for different methods with HH-RLHF dataset.}
    \label{fig:gradient_loss_S}
\end{figure*}

\begin{figure*}[h]
    \centering
    \includegraphics[width=\linewidth]{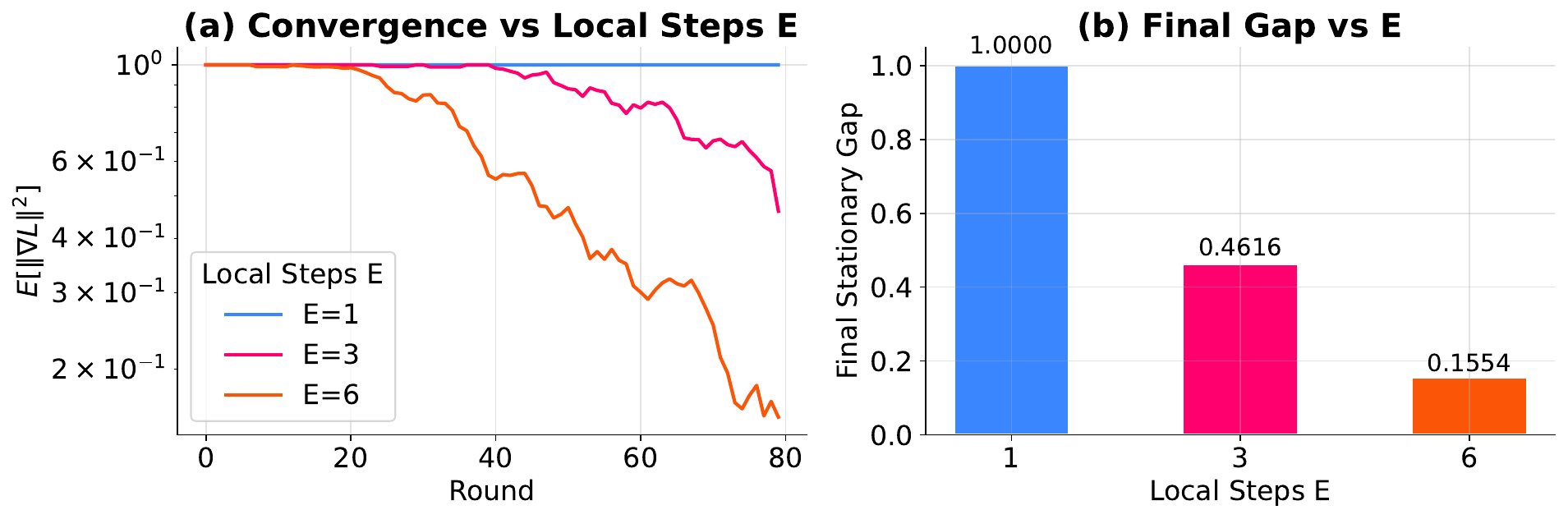}
    \caption{Impact of local steps on performance with HH-RLHF.}
    \label{fig:local_steps_S}
\end{figure*}

\textbf{Local steps $E$.} In Figure~\ref{fig:local_steps_S}, $E=1$ remains flat at $\sim$1.0 throughout all 80 rounds, $E=3$ converges moderately, and $E=6$ achieves the lowest final stationary gap ($E=1: 1.000, E=3: 0.462, E=6: 0.155$). This monotone improvement with $E$ is fully consistent with Theorem~\ref{theorem_1} in the computation-dominated regime: at $\eta=0.0001$, the drift floor $\eta^2E\kappa^2 \approx (1e-8)E\kappa^2$ is negligible for all $E$, while the convergence term $1/(\eta ER)$ is directly reduced by larger $E$, making more local gradient work per round strictly beneficial. The complete failure of $E=1$ to converge highlights an important practical implication: at small $\eta$, a single local step per round provides insufficient gradient signal to make meaningful progress within 80 rounds. This reinforces the guideline derivable from Theorem~\ref{theorem_1} that, in communication-efficient regimes with small $\eta$, $E$ should be set generously.

\begin{figure*}[h]
    \centering
    \includegraphics[width=\linewidth]{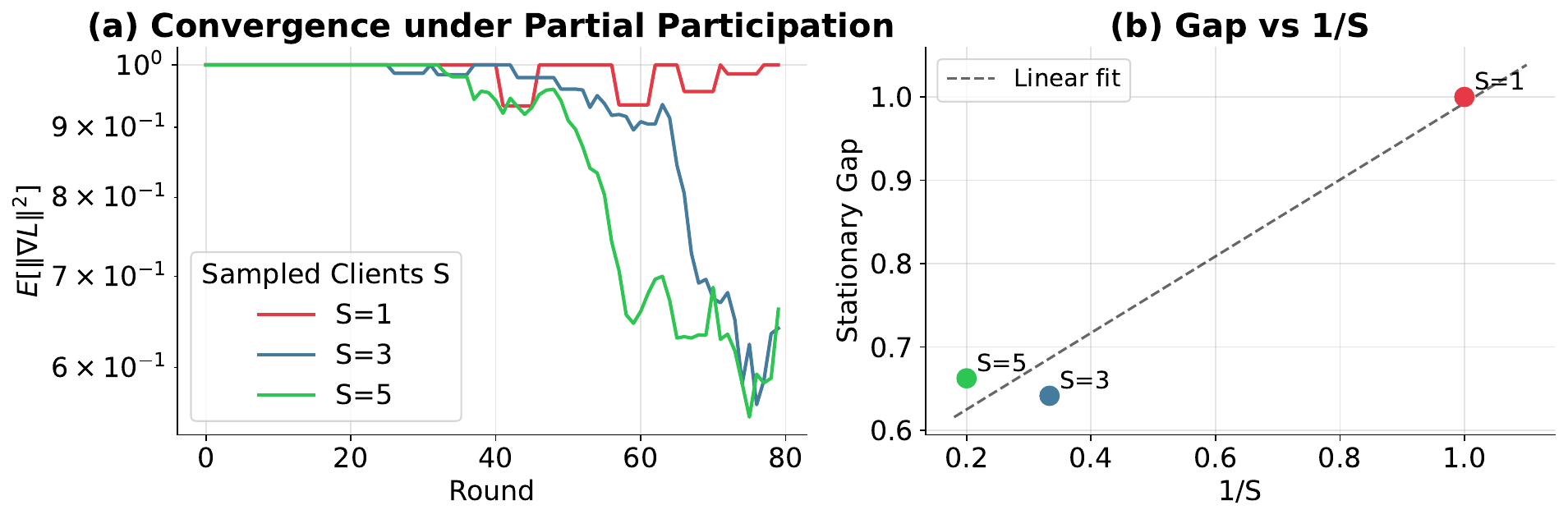}
    \caption{Impact of participation on performance with HH-RLHF.}
    \label{fig:participation_S}
\end{figure*}

\textbf{Participation $S$.} In Figure~\ref{fig:participation_S}, $S=5$ converges most smoothly to the lowest final gradient norm, $S=3$ is intermediate, and $S=1$ shows the highest variability with occasional upward spikes. Panel (b) provides direct quantitative validation of Theorem~\ref{theorem_1}'s $\zeta^2_g/S$ variance floor: final stationary gaps $S=5: \sim0.67, S=3: \sim0.65, S=1: \sim1.0$ lie close to the linear fit in the $1/S$ scatter plot. The positive slope of the linear fit represents the empirical $\zeta^2_g$ coefficient, and the approximate linearity confirms that the dominant residual gap is the variance floor rather than heterogeneity drift. The fluctuation behavior of $S=1$ in panel (a) is also consistent with theory: with a single client per round, gradient estimates carry high variance, producing occasional upward steps before the overall trend descends — precisely the behavior predicted by the large $\zeta_g^2/S$ term. 

\begin{figure*}[h]
    \centering
    \includegraphics[width=\linewidth]{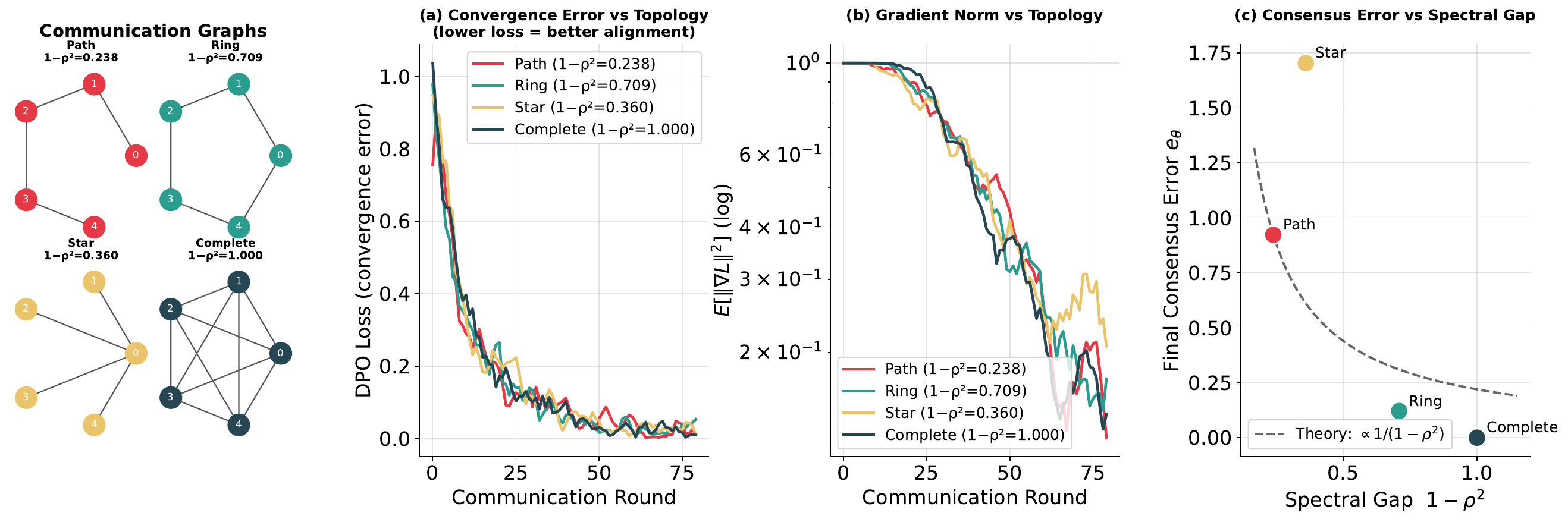}
    \caption{Impact of topology on performance with HH-RLHF.}
    \label{fig:topology_S}
\end{figure*}

\textbf{DecDPO topology.} In Figure~\ref{fig:topology_S}, all four topologies converge from an initial DPO loss of $\sim0.8-1.0$, with the complete graph achieving the lowest final loss and path the highest in panel (a). Panel (b) shows that gradient norm steady-state floors follow the ordering complete $<$ ring $\approx$ path $<$ star, consistent with Theorem~\ref{theorem_4}'s prediction that the floor scales as $\eta^2(\zeta^2_g+\kappa^2)/(1-\rho^2)$. Panel (c) provides the most direct test of Theorem~\ref{theorem_4}: consensus error $\mathfrak{e}_\theta$ follows the $1/(1-\rho^2)$ theoretical curve closely across path ($\sim$1.00), ring ($\sim$0.18), and complete ($\sim$0.04), spanning over an order of magnitude. The star anomaly ($\mathfrak{e}_\theta\approx 1.72$, above path) is a finite-$N=5$ structural effect: the hub-and-spoke topology creates asymmetric mixing weights where four spoke nodes remain far from the network mean while the hub pulls toward it, inflating $\mathfrak{e}_\theta$ beyond what the spectral gap alone predicts. This effect is expected to diminish as $N$ grows and does not contradict the theory, which holds for symmetric doubly stochastic mixing matrices at any fixed $N$.

\begin{figure*}[h]
    \centering
    \includegraphics[width=\linewidth]{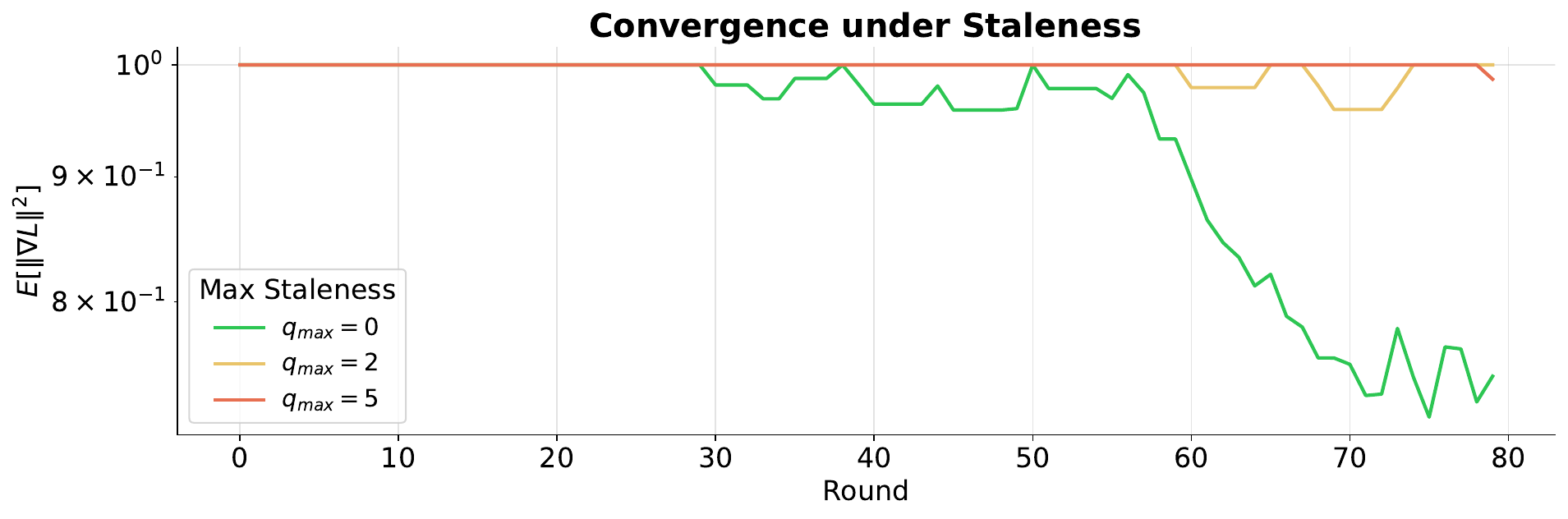}
    \caption{Impact of staleness on performance with HH-RLHF.}
    \label{fig:staleness_S}
\end{figure*}

\textbf{Staleness.} In Figure~\ref{fig:staleness_S}, only $q_\text{max}$=0 (synchronous) shows meaningful convergence, descending to $\sim0.73-0.77$ by round 80. Both $q_\text{max}=2$ and $q_\text{max}=5$ remain essentially flat at $\sim$1.0 throughout. The qualitative ordering $q_\text{max}=0$ $>$ $q_\text{max}=2$ $>$ $q_\text{max}=5$ is directionally consistent with Theorem~\ref{theorem_3}'s $C_qq_\text{max}$ staleness penalty. However, the complete stalling of $q_\text{max}\geq 2$ is more severe than the theorem's additive penalty alone would suggest, indicating that stale updates at small $\eta$ introduce gradient direction inconsistencies that dominate the small learning signal per round — an effect that Assumption~\ref{assumption_6}'s bounded drift model does not fully capture. We flag this as an important gap between the theoretical model and practical implementation, which could be refined by adjusting the staleness mechanism.

\textbf{Cross-Dataset Summary: HH-RLHF vs. SHP.}
The HH-RLHF results are strongly consistent with SHP across all experimental conditions, providing robust cross-dataset validation of the theoretical framework. The participation $S$ ablation is the most dataset-invariant result: both datasets show a clear linear relationship between final stationary gap and $1/S$, confirming the $\zeta_g^2/S$ variance floor, with SHP exhibiting a steeper slope (larger empirical $\zeta^2_g$) attributable to its higher crowd-sourced preference noise. The topology ablation also replicates faithfully: both datasets show the same complete $<$ ring $<$ path ordering in consensus error, with the star anomaly appearing consistently above path in both cases and the $1/(1-\rho^2)$ theoretical curve fitting both with comparable fidelity. For the local steps $E$ ablation, both datasets show the same monotone improvement with $E$ and the same qualitative computation-dominated behavior, with HH-RLHF yielding slightly larger final gaps than SHP. The comparative study shows narrower std bands on HH-RLHF than SHP, consistent with HH-RLHF's curated annotations producing lower gradient variance than SHP's crowd-sourced scores. Overall, the strong quantitative and qualitative consistency across both datasets confirms that the key DPO-specific scaling laws are dataset-invariant structural properties rather than artifacts of a particular preference distribution.

\begin{figure*}[h]
    \centering
    \includegraphics[width=\linewidth]{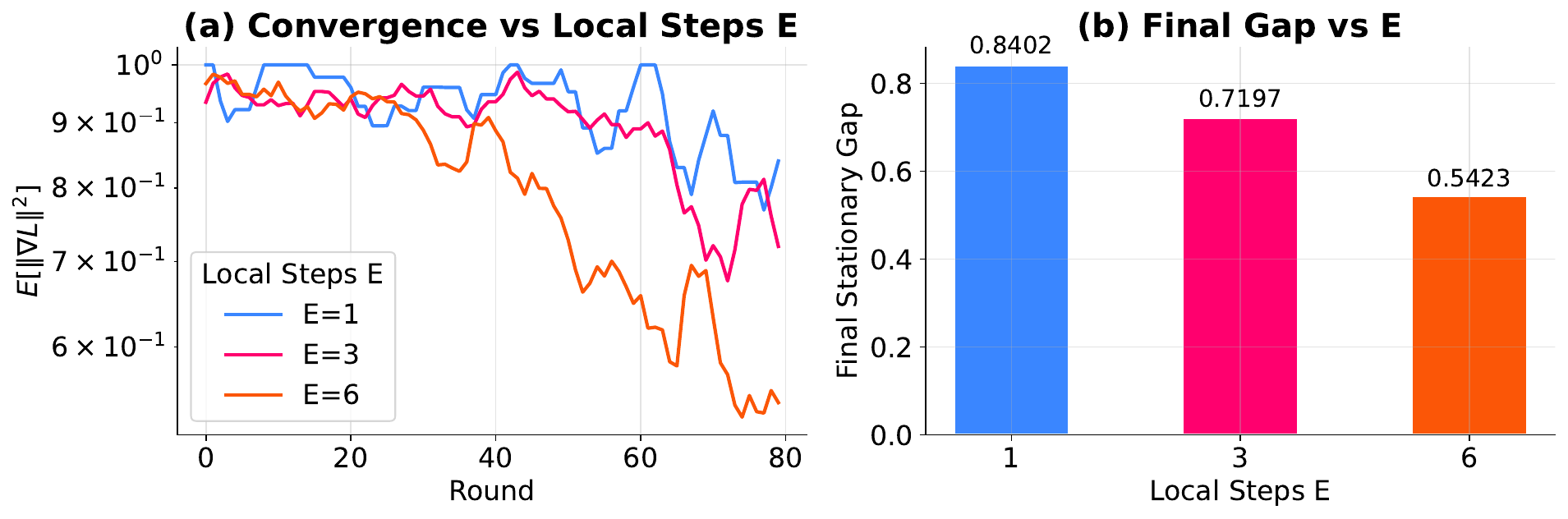}
    \caption{Impact of local steps on performance with HH-RLHF and learning rate equal to $0.01$.}
    \label{fig:local_steps_S_0.01}
\end{figure*}

\begin{figure*}[h]
    \centering
    \includegraphics[width=\linewidth]{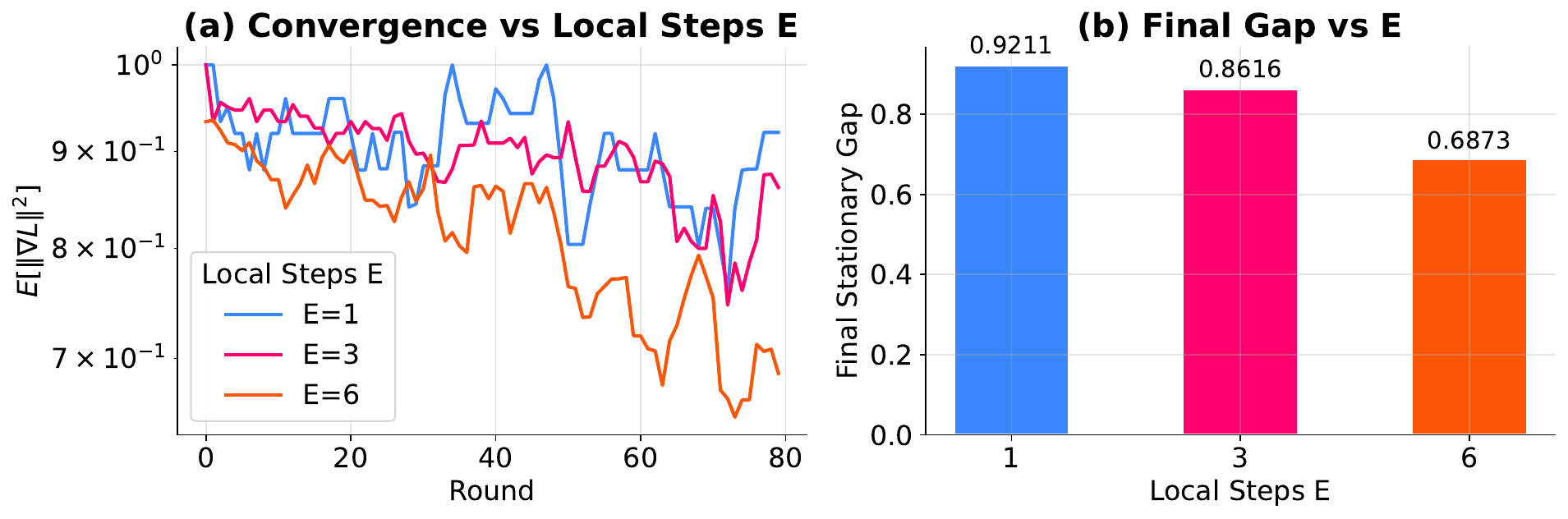}
    \caption{Impact of local steps on performance with HH-RLHF and learning rate equal to $0.1$.}
    \label{fig:local_steps_S_0.1}
\end{figure*}
\textbf{Local Steps $E$ vs. Learning Rate $\eta$.} We also investigate dominance between $\frac{1}{\eta E}$ and $\eta^2E\kappa^2+\eta^2E^2\zeta^2_g$ in Theorem~\ref{theorem_1}.
\begin{itemize}
    \item $\mathbf{\eta=0.01}$: In Figure~\ref{fig:local_steps_S_0.01}, increasing $E$ from 1 to 6 consistently improves convergence (lower final gap, faster drop in gradient norm). The beneficial term $\frac{(\mathcal{L}_{\theta^0}-\mathcal{L}^*)}{\eta ER}$ dominates because harmful terms $\eta^2E\kappa^2+\eta^2E^2\zeta^2_g$ scale with $\eta^2$ and are negligible. 
    \item $\mathbf{\eta=0.1}$: The final gap decreases as $E$ increases from 1 to 6 (no U‑shape) as shown in Figure~\ref{fig:local_steps_S_0.1}, meaning the beneficial term still dominates over the harmful drift/variance terms within this $E$ range. However, all gaps are larger than those at $\eta=0.01$ such that raising $\eta$ worsens overall convergence, consistent with harmful terms growing as $\eta^2$. The gradient norm curves overlap significantly, and in some regimes $E=3$ yields higher gradient norm than $E=1$. This indicates that for $\eta=0.1$, the harmful effects (client drift and variance) are not negligible; they occasionally push performance backward, even though the final gap after many rounds still favors larger $E$ (e.g., 10, 20), which is a direction we leave for future work.
\end{itemize}

\end{document}